%% file: paper_acm.tex
\documentclass[sigconf]{acmart}

\input{lib/_shared_imports.tex}

\copyrightyear{2023}
\acmYear{2023}
\setcopyright{rightsretained}
\acmConference[KDD '23]{Proceedings of the 29th ACM SIGKDD Conference on Knowledge Discovery and Data Mining}{August 6--10, 2023}{Long Beach, CA, USA}
\acmBooktitle{Proceedings of the 29th ACM SIGKDD Conference on Knowledge Discovery and Data Mining (KDD '23), August 6--10, 2023, Long Beach, CA, USA}
\acmDOI{10.1145/3580305.3599268}
\acmISBN{979-8-4007-0103-0/23/08}

\ccsdesc[500]{Computing methodologies~Neural networks}
\ccsdesc[500]{Computing methodologies~Learning latent representations}
\ccsdesc[300]{Networks~Network algorithms}

\begin{document}

\title{\paperTitle{}}

\author{William Shiao}
\orcid{0000-0001-5813-2266}
\email{wshia002@ucr.edu}
\affiliation{%
  \institution{University of California, Riverside}
  \city{Riverside}
  \state{CA}
  \country{USA}
}

\author{Uday Singh Saini}
\orcid{0000-0002-8561-5527}
\email{usain001@ucr.edu}
\affiliation{%
  \institution{University of California, Riverside}
  \city{Riverside}
  \state{CA}
  \country{USA}
}

\author{Yozen Liu}
\orcid{0000-0002-2107-504X}
\email{yliu2@snap.com}
\affiliation{%
  \institution{Snap Inc.}
  \city{Santa Monica}
  \state{CA}
  \country{USA}
}

\author{Tong Zhao}
\orcid{0000-0001-7660-1732}
\email{tzhao@snap.com}
\affiliation{%
  \institution{Snap Inc.}
  \city{Seattle}
  \state{WA}
  \country{USA}
}

\author{Neil Shah}
\orcid{0000-0003-3261-8430}
\email{nshah@snap.com}
\affiliation{%
  \institution{Snap Inc.}
  \city{Seattle}
  \state{WA}
  \country{USA}
}

\author{Evangelos E. Papalexakis}
\email{epapalex@cs.ucr.edu}
\orcid{0000-0002-3411-8483}
\affiliation{%
  \institution{University of California, Riverside}
  \city{Riverside}
  \state{CA}
  \country{USA}
}

\renewcommand{\shortauthors}{Shiao et al.}

\begin{abstract}
  \input{sections/000abstract}
\end{abstract}

\keywords{\paperKeywords{}}

\maketitle
\input{sections/_body.tex}

\begin{acks}
  \input{sections/500acks.tex}
\end{acks}

\bibliographystyle{ACM-Reference-Format}
\bibliography{bib/vagelis_refs,bib/refs.bib}

\clearpage
\appendix
\input{sections/900appendix.tex}

\end{document}

%% file: lib/_shared_imports.tex
\input{lib/config.tex}

\input{lib/math_commands.tex}

\input{lib/macros.tex}

\usepackage{duckuments}     %
\usepackage{amsmath,graphicx,amsthm}
\usepackage[ruled,linesnumbered]{algorithm2e}
\usepackage{algorithmic}
\usepackage{listings}
\usepackage{cuted}
\usepackage{paralist} %
\usepackage{color}
\usepackage{rotating}
\usepackage{hyperref}
\usepackage[mathscr]{eucal} %
\usepackage{amsbsy} %
\usepackage{subcaption}
\usepackage{booktabs}
\usepackage{xcolor}
\usepackage{tcolorbox}
\usepackage{multirow}
\usepackage{todonotes}
\usepackage{epstopdf}
\usepackage{balance}
\usepackage{tablefootnote}
\usepackage{empheq} %
\usepackage{moresize}
\usepackage{enumitem}
\usepackage{graphicx}
\usepackage[utf8]{inputenc}
\usepackage{nicematrix}
\usepackage{wrapfig} %
\usepackage[capitalize,noabbrev,nameinlink]{cleveref}
\setlist{nolistsep}
\usepackage{pifont} %

\DeclareCaptionType{copyrightbox}
\usepackage{url}

\newcounter{ALC@tempcntr}%

\hypersetup{
colorlinks=true,
linkcolor=red,
anchorcolor=blue,
citecolor=brown
}

%% file: lib/config.tex
\newcommand{\paperKeywords}{self-supervised graph representation learning, clustering}

\newcommand{\rawMethodName}{CARL-G}
\newcommand{\method}{\mbox{%
    \small\textsf{\rawMethodName{}}}\xspace}

\newcommand{\paperTitle}{\rawMethodName{}: Clustering-Accelerated Representation Learning on Graphs}

%% file: lib/math_commands.tex
\usepackage{amsmath,amsfonts,bm}

\def\eqref#1{equation~\ref{#1}}

\def\1{\bm{1}}

\def\va{{\bm{a}}}
\def\vb{{\bm{b}}}

\def\vh{{\bm{h}}}

\def\vx{{\bm{x}}}

\def\mA{{\bm{A}}}

\def\mD{{\bm{D}}}

\def\mH{{\bm{H}}}

\def\mW{{\bm{W}}}
\def\mX{{\bm{X}}}

\def\mZ{{\bm{Z}}}

\DeclareMathAlphabet{\mathsfit}{\encodingdefault}{\sfdefault}{m}{sl}
\SetMathAlphabet{\mathsfit}{bold}{\encodingdefault}{\sfdefault}{bx}{n}

\def\gC{{\mathcal{C}}}

\def\gE{{\mathcal{E}}}

\def\gG{{\mathcal{G}}}

\def\gL{{\mathcal{L}}}

\def\gU{{\mathcal{U}}}
\def\gV{{\mathcal{V}}}

\newcommand{\E}{\mathbb{E}}

\newcommand{\R}{\mathbb{R}}

\DeclareMathOperator*{\argmin}{arg\,min}

%% file: lib/macros.tex
\long\def\IfNoTokA #1\IfNoTokB \IfNoTokB {}
\long\def\IfNoTokB \IfNoTokC #1#2{#1}    
\long\def\IfNoTokC           #1#2{#2}%
\makeatletter
\newcommand*{\tp}{%
  {\mathpalette\@transpose{}}%
}
\newcommand*{\@transpose}[2]{%
  \raisebox{\depth}{$\m@th#1\intercal$}%
}
\makeatother

\long\def\IfNoTokens #1{\IfNoTokA\IfNoTokB #1\IfNoTokB\IfNoTokB\IfNoTokC }

\definecolor{aqua}{rgb}{0.0, 1.0, 1.0}

\newcommand*{\scale}[2][4]{\scalebox{#1}{$#2$}}

\newcommand{\methodname}{\method{}}

\newcommand{\wtodo}[1]{{\textcolor{blue}{\textbf{[\IfNoTokens{#1}{TODO}{TODO: #1}]}}}}

\newcommand{\cmark}{\ding{51}}%
\newcommand{\xmark}{\ding{55}}%

\newcommand{\tDist}{\textsc{Dist}}

\newcommand{\neigh}{{\mathscr{N}}}

\newcommand{\Ebr}[1]{\E \left[ #1 \right]}
\newcommand{\Ebrs}[2]{\E_{#1} \left[ #2 \right]}

\newcommand{\coauthorcs}{\texttt{Coauthor-CS}}
\newcommand{\coauthorphysics}{\texttt{Coauthor-Physics}}
\newcommand{\amazoncomputers}{\texttt{Amazon-Computers}}
\newcommand{\amazonphotos}{\texttt{Amazon-Photos}}
\newcommand{\wikics}{\texttt{Wiki-CS}}

\theoremstyle{plain}
\newtheorem{claim}{Claim}

\newcommand{\gCent}{{\boldsymbol{\mu}}}
\newcommand{\lCent}[1]{{\boldsymbol{\mu}}_{#1}}
\newcommand{\vrc}{\textsc{VRC}}
\newcommand{\svrc}{\textsc{vrc}} %
\newcommand{\meanSil}{\textsc{ms}}
\newcommand{\sil}{\textsc{Sil}}
\newcommand{\pcAssign}{{\gU}} %

\newcommand{\silTar}{\tau_{\sil}} %
\newcommand{\vrcTar}{\tau_{\svrc}} %
\newcommand{\simp}{\textsc{sim}} %

\newcommand{\boldMethod}{\textbf{\method{}}}
\newcommand{\silMethod}{\ensuremath{\method{}_\sil{}}}
\newcommand{\simpMethod}{\ensuremath{\method{}_\simp{}}}
\newcommand{\vrcMethod}{\ensuremath{\method{}_\vrc{}}}
\newcommand{\numClusters}{\ensuremath{c}}
\newcommand{\arxivOnly}[1]{}

\newcommand{\kmeans}{$k$-means}

%% file: sections/000abstract.tex
Self-supervised learning on graphs has made large strides in achieving great performance in various downstream tasks. However, many state-of-the-art methods suffer from a number of impediments, which prevent them from realizing their full potential. For instance, contrastive methods typically require negative sampling, which is often computationally costly. While non-contrastive methods avoid this expensive step, most existing methods either rely on overly complex architectures or dataset-specific augmentations. In this paper, we ask: \textit{Can we borrow from classical unsupervised machine learning literature in order to overcome those obstacles?} 
Guided by our key insight that the goal of distance-based clustering closely resembles that of contrastive learning: both attempt to pull representations of similar items together and dissimilar items apart.
As a result, we propose \method{} --- a novel \textit{clustering-based} framework for graph representation learning that uses a loss inspired by Cluster Validation Indices (CVIs), i.e., 
internal measures of cluster quality (no ground truth required). \method{} is adaptable to different clustering methods and CVIs, and we show that with the right choice of clustering method and CVI, \method{} outperforms node classification baselines on 4/5 datasets with up to a \textbf{79\texttimes{}} training speedup compared to the best-performing baseline. \method{} also performs at par or better than baselines %
in node clustering and similarity search tasks, training up to \textbf{1,500\texttimes{}} faster than the best-performing baseline. Finally, we also provide theoretical foundations for the use of CVI-inspired losses in graph representation learning. 

%% file: sections/_body.tex
\input{sections/010introduction}
\input{sections/015prelim}

\input{sections/030prob_method}

\input{sections/040experiments}
\input{sections/050related}

\input{sections/060conclusions}

%% file: sections/010introduction.tex
\section{Introduction}
\label{sec:intro}
Graphs can be used to represent many different types of relational data, including  chemistry graphs~\cite{chen2019graph,guo2021few}, social networks~\cite{randomGraphs,coraCiteseer}, and traffic networks~\citep{derrow2021eta, tang2020knowing}. Graph Neural Networks (GNNs) have been effective in modeling graphs across a variety of tasks, such as for recommendation systems~\citep{pinSAGE,he2020lightgcn,sankar2021graph, tang2022friend,fan2022graph,zhao2022learning}, graph generation~\citep{graphrnn,fan2019labeled,shiao2021adversarially}, and node classification~\citep{grace,bgrl,seal,zhao2021data,jin2022empowering,han2022mlpinit}. These GNNs are traditionally~\citep{gcn} trained with a supervised loss, which requires labeled data that is often expensive to obtain in real-world scenarios. Graph self-supervised learning (SSL), a recent area of research, attempts to solve this by learning multi-task representations without labeled data~\cite{grace,bgrl,jin2021automated,gbt,ju2023multi}.

\begin{figure}[t]
    \centering
    \includegraphics[width=0.9\columnwidth]{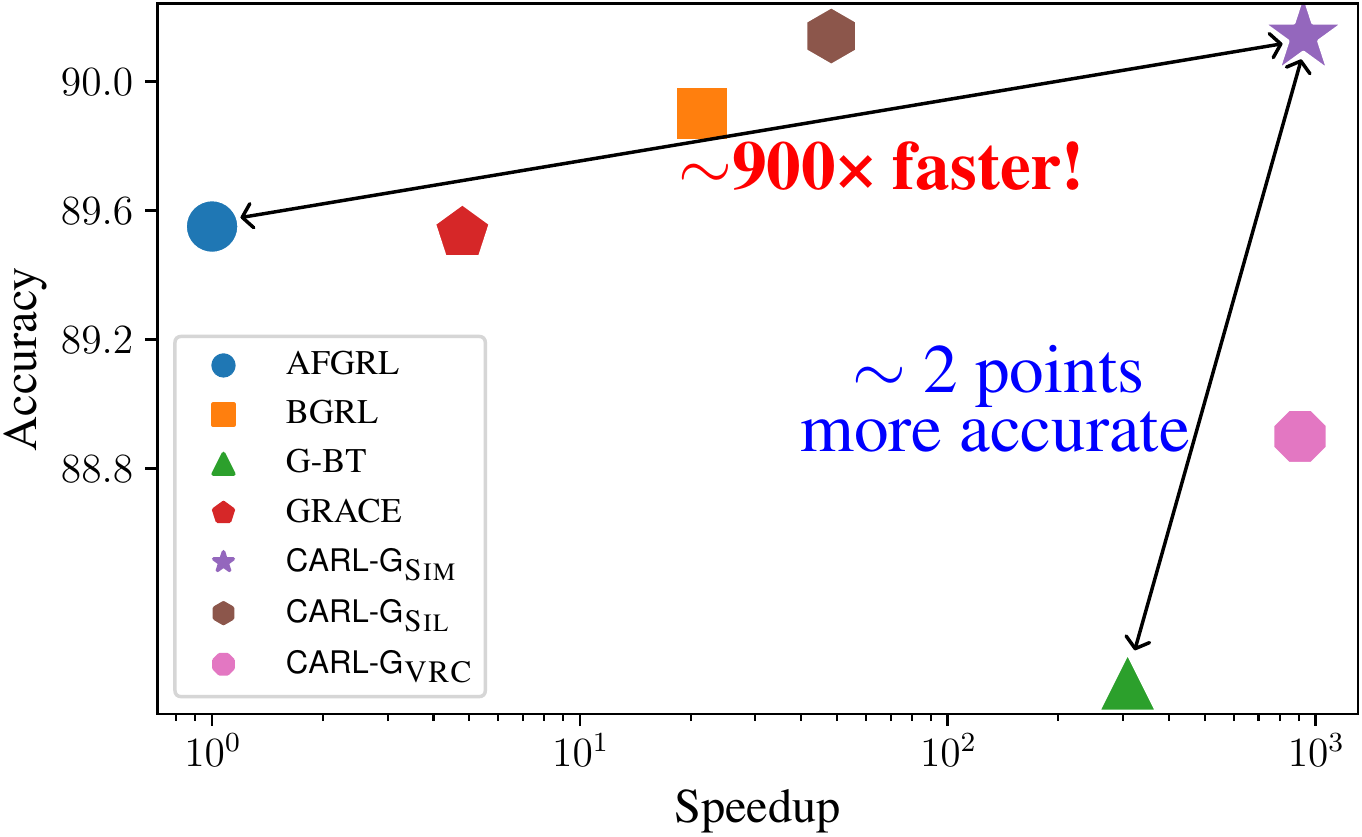}
    \caption{Comparison of our proposed methods with other baselines with respect to node classification accuracy and speedup on the \amazonphotos{} dataset. See \cref{fig:run_vs_acc} for results on the other datasets.}
    \label{fig:jewel}
    \Description{Image showing that our proposed method outperforms other baselines in terms of both accuracy and speed.}
\end{figure}

\input{tables/salesman.tex}

Most of these existing graph SSL methods can be grouped into either contrastive or non-contrastive SSL. Contrastive learning pulls the representations of similar (``positive'') samples together and pushes the representations of dissimilar (``negative'') samples apart. In the case of graphs, this often means either pulling the representations of a node and its neighbors together~\citep{graphsage} or pulling the representations of the same node across two different augmentations together~\cite{you2020graph}. Graph contrastive learning methods typically use non-neighbors as negative samples~\cite{grace,pinSAGE}, which can be costly.
Non-contrastive learning~\citep{tbgrl,bgrl,gbt,afgrl} avoids this step by only pulling positive samples together while employing strategies to avoid collapse.

However, all of these methods have some key limitations. %
Contrastive methods rely on a negative sampling step, which has an expensive quadratic runtime~\cite{bgrl} and requires careful tuning~\cite{yang2020understanding}. Non-contrastive methods often have complex architectures (ex. extra encoder with exponentially updated weights~\cite{bgrl,afgrl,selfgnn}) and/or rely heavily on augmentations~\cite{gbt,ccaSSG,bgrl,gda}. \citet{afgrl} shows that augmentations can change the semantics of underlying graphs, especially in the case of molecular graphs (where perturbing a single edge can create an invalid molecule). %

Upon further inspection, we observe that the behavior of contrastive and non-contrastive methods is somewhat similar to that of distance-based clustering~\cite{xu2015comprehensive}---both attempt to pull together similar nodes/samples and push apart dissimilar ones. The primary advantage of using clustering over negative sampling is that we can work directly in the smaller embedding space, preventing expensive negative sampling over the graph. Furthermore, there have been decades of research exploring the theoretical foundations of clustering methods and many different metrics have been proposed to evaluate the quality of clusters~\cite{silhouette,vrc,dbIndex,dunnIndex}. These metrics have been dubbed Cluster Validation Indices (CVIs)~\cite{arbelaitz2013extensive}. In this work, we ask the following question: \textit{Can we leverage well-established clustering methods and CVIs to create a flexible, fast, and effective GNN framework to learn node representations without any labels?}

It is worth emphasizing that our goal is \textit{not node clustering} directly---it is self-supervised graph representation learning. The goal is to develop a general framework that is capable of learning node embeddings for various tasks, including
node classification, node clustering, and node similarity search. 
There exists some similar work. DeepCluster~\cite{caron2018deep} trains a Convolutional Neural Network (CNN) with a supervised loss on pseudo-labels generated by a clustering method to learn image embeddings.

AFGRL~\cite{afgrl} uses clustering to select positive samples in lieu of augmentations for graph representation learning and applies the general BGRL~\cite{bgrl} framework to push those representations together. SCD~\cite{vskrlj2020embedding} searches over different hyperparameters to obtain a clustering where the silhouette score is maximized. However, to the best of our knowledge, there is no existing work in self-supervised representation learning that directly optimizes for CVIs, which, as we elaborate below, presents us with tremendous potential in advancing and accelerating the state of the art.

We fill this gap by proposing the novel idea of using Cluster Validation Indices (CVIs) directly as our loss function for a neural network. In conjunction with advances in clustering methods~\cite{schubert2019faster,minibatchKmeans,lenssen2022clustering,pelleg1999accelerating}, CVIs have been improved over the years as measures of cluster quality after performing clustering and have been shown for almost 5 decades to be effective for this purpose~\cite{silhouette,schubert2022stop,arbelaitz2013extensive,vrc}.

Our proposed method, \method{}, has several advantages over existing graph SSL methods by virtue of CVI-inspired losses. First, \method{} generally outperforms other graph SSL methods in node clustering, clustering, and node similarity tasks (see \cref{tab:classification,tab:similarity_search,tab:clustering}). 
Second, \method{} does not require the use of graph augmentations --- which are required by many existing graph contrastive and non-contrastive methods~\citep{grace,selfgnn,bgrl,afgrl,gbt,ccaSSG} and can inadvertently alter graph semantics~\cite{afgrl}. %
Third, \method{} has a relatively simple architecture compared to the dual encoder architecture of leading non-contrastive methods~\cite{bgrl,selfgnn,afgrl}, drastically reducing the memory cost of our framework. 
Fourth, we provide theoretical analysis that shows the equivalence of some CVI-based losses and a well-established (albeit expensive) contrastive loss. Finally, \method{} is sub-quadratic with respect to the size of the graph and much faster than the baselines, with up to a 79\texttimes{} speedup on \coauthorcs{} over BGRL~\cite{bgrl} (the best-performing baseline).

Our contributions can be summarized as follows:

\begin{itemize}[noitemsep,topsep=0em]
    \item We propose \method{}, the first (to the best of our knowledge) framework to use a Cluster Validation Index (CVI) as a neural network loss function.
    \item We propose 3 variants of \method{} based on different CVIs, each with its own advantages and drawbacks.
    \item We extensively evaluate \simpMethod{} --- the best all-around performer --- across 5 datasets and 3 downstream tasks, where it generally outperforms baselines. 
    \item We provide theoretical insight on \simpMethod{}'s success.
    \item We benchmark \simpMethod{} against 4 state-of-the-art models and show that it is up to 79 \texttimes{} faster with half the memory consumption (with the same encoder size) compared to the best-performing node classification baseline.
\end{itemize}%

%% file: tables/salesman.tex
\begin{table*}[ht]
\centering
\resizebox{\linewidth}{!}{%
\begin{tabular}{r|c||c|c|c|c|c|}
& \multicolumn{1}{c||}{\textbf{Proposed}} & \multicolumn{5}{c|}{\textbf{Baseline Methods}} \\
\cline{2-7}
\textbf{} %
& \multicolumn{1}{c||}{\boldMethod{}*} %
& \multicolumn{1}{c|}{\textbf{AFGRL}~\cite{afgrl}} %
& \multicolumn{1}{c|}{\textbf{BGRL}~\cite{bgrl}} %
& \multicolumn{1}{c|}{\textbf{G-BT}~\cite{gbt}} %
& \multicolumn{1}{c|}{\textbf{GRACE}~\cite{grace}} %
& \multicolumn{1}{c|}{\textbf{MVGRL}~\cite{mvgrl}} \\ %
\hline
Avoids Negative Sampling       & \cmark & \cmark & \cmark & \cmark & \xmark & \xmark \\
Augmentation-Free              & \cmark & \cmark & \xmark & \xmark & \xmark & \xmark \\
Single Encoder                 & \cmark & \xmark & \xmark & \cmark & \xmark & \xmark \\
Single Forward Pass per Epoch  & \cmark & \xmark & \xmark & \xmark & \xmark & \xmark \\
\hline
\end{tabular}
}
\caption{Comparison of different self-supervised graph learning methods. *: We use \simpMethod{} as the representative method since it is the best-performing across all of the criteria.
}
\label{table:salesman}
\vspace{-0.25in}
\end{table*}

%% file: sections/015prelim.tex
\section{Preliminaries}
\label{sec:preliminaries}

\begin{figure*}[t]
    \centering
    \includegraphics[width=\textwidth]{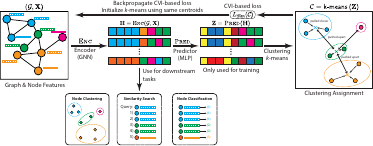}
    \caption{\methodname{} architecture diagram. We describe the method in detail in \cref{sec:method}.}
    \label{fig:arch}
    \Description{Our architecture diagram, which consists of an encoder, predictor, and clustering step. The embeddings are then used for downstream tasks.}
\end{figure*}

\paragraph{Notation} We denote a graph as $G = (\gV, \gE)$. $\gV$ is the set of $n$ nodes (i.e., $n = |\gV|$) and $\gE \subseteq \gV \times \gV$ is the set of edges. We denote the node-wise feature matrix as $\mX \in \R^{n \times f}$, where $f$ is the dimension of raw node features, and its $i$-th row $\vx_i$ is the feature vector for the $i$-th node. We denote the binary adjacency matrix as $\mA \in \{0, 1\}^{n \times n}$ and the learned node representations as $\mH \in \R^{n \times d}$, where $d$ is the size of latent dimension, and $\vh_i$ is the representation for the $i$-th node. Let $\neigh(u)$ be a function that returns the set of neighbors for a given node $u$ (i.e., $\neigh(u) = \{ v | (u,v) \in \gE \}$).%

Let $\gC$ be the set of clusters, $\gC_i \subseteq \gV$ be the set of nodes in the $i$-th cluster, and $c = |\gC|$ be the number of clusters. For ease of notation, let $\gU \in [1, c]^n$ be the set of cluster assignments, where $\gU_u$ is the cluster assignment for node $u$. Let $\gCent = \frac{1}{c} \sum_{u \in \gV} \vh_u$ be the global centroid and $\lCent{i} = \frac{1}{|\gC_i|} \sum_{u \in \gC_i} \vh_u$ be the centroid for the $i$-th cluster.%
\subsection{Graph Neural Networks}
\label{subsec:gnns}

A Graph Neural Network (GNN)~\citep{gcn,graphsage,zhang2020deep} typically performs message-passing along the edges of the graph. Each iteration of the GNN can be described as follows~\cite{gnn}:%
\vspace{-0.16em}
\begin{equation}
    \vh_u^{(k+1)} = \textsc{Update}^{(k)} \left(
        \vh_u^{(k)}, \textsc{Aggregate}^{(k)} (
            \{  \vh_v^{(k)}, \forall v \in \neigh (u) \}
        )
    \right) \thinspace ,
\end{equation}%
where \textsc{Update} and \textsc{Aggregate} are differentiable functions, and $\vh^{(0)}_u = \vx_u$. In this work, we opt for simplicity and use Graph Convolutional Networks (GCNs)~\citep{gcn} as the default GNN. These are GNNs where \textsc{Update} consists of a single MLP layer, and \textsc{Aggregate} is the mean of a node's representation with its neighbors. Formally, each iteration of the GCN can be written as:%
\begin{equation}
    \vh_u^{(k+1)} = \sigma \left(
        \mW^{(k+1)} \sum_{v\in \neigh(u) \cup \{ u \}} \frac{\vh_v}{\sqrt{| \neigh(u) | \cdot | \neigh(v) |}}
    \right) \thinspace.
\end{equation}

\subsection{Cluster Validation Indices}
\label{subsec:cvis}
Clustering is a class of unsupervised methods that aims to partition the input space into multiple groups, known as clusters. The goal of clustering is generally to maximize the similarity of points within each cluster while minimizing the similarity of points between clusters~\cite{xu2015comprehensive}. In this work, we focus on centroid-based clustering algorithms like $k$-means~\cite{kmeans} and $k$-medoids~\cite{kmedoids}.

Cluster Validation Indices (CVIs)~\cite{arbelaitz2013extensive} estimate the quality of a partition (i.e., clustering) by measuring the compactness and separation of the clusters without knowledge of the ground truth clustering. Note that these are different from metrics like Normalized Mutual Information (NMI)~\cite{nmi} or the Rand Index~\cite{rand}, which require ground truth information of cluster labels. Many different CVIs have been proposed over the years and extensively evaluated~\cite{arbelaitz2013extensive,schubert2022stop}.

\citet{arbelaitz2013extensive} extensively evaluated 30 different CVIs over a wide variety of datasets and found that the Silhouette~\cite{silhouette}, Davies-Bouldin*~\cite{dbIndex}, and Calinski-Harabasz (also known as the VRC: Variance Ratio Criterion)~\cite{vrc} indices perform best across 720 different synthetic datasets. The VRC has also been shown to be effective in determining the number of clusters for clustering methods~\cite{schubert2022stop,vrc,milligan1985examination,clusterAnalysisBook}. As such, we focus on the silhouette index (the best-performing CVI) and VRC (an effective CVI --- especially for choosing the number of clusters) in this work.

\subsubsection{Silhouette}
\label{subsec:silhouette}
The silhouette index computes the ratio of intra-cluster distance with respect to the inter-cluster distance of itself with its nearest neighboring cluster. It returns a value in [-1, 1], where a value closer to 1 signifies more desired and better distinguishable clustering. 
The silhouette index~\cite{silhouette} is defined as $\sil{} (\gC) = \frac{1}{n} \sum_{u \in \gV} s(u)$, where:%
\begin{align}
\label{eqn:silhouette}
s(u) = \frac{b(u) - a(u)}{\max \{ a(u), b(u) \}} \thinspace ,
\end{align}%
and%
\begin{align}
a(u) &= \frac{1}{| \gC_{\gU_u} | - 1} \sum_{v \in (\gC_{\gU_u} - \{ u\})} \tDist(\vh_u, \vh_v) \thinspace , \\
b(u) &= \min_{i \neq \gU_u} \frac{1}{|\gC_i|} \sum_{v \in \gC_i} \tDist(\vh_u, \vh_v) \thinspace .
\end{align}%
The runtime of computing the silhouette index for a given node is $O(n)$, which can be expensive if calculated over all nodes. We discuss this issue and a modified solution later in \cref{subsec:training}.
\subsubsection{Variance Ratio Criterion} The VRC~\cite{vrc} computes a ratio between its intra-cluster variance and its inter-cluster variance. Its intra-cluster variance is based on the distances of each point to its centroid, and its inter-cluster variance is based on the distance from each cluster centroid to the global centroid. Formally,
\begin{equation}
\vrc{}(\gC) = \frac{n - c}{c - 1} \frac{\sum_{\gC_k \in \gC} |\gC_k| \tDist{}(\lCent{k}, \gCent{})}{ \sum_{\gC_k \in \gC} \sum_{u \in \gC_k} \tDist{}(\vh_u, \gCent{}) } \thinspace .
\end{equation}%
For the purposes of this paper, we use Euclidean distance, i.e., $\tDist(\va, \vb) = \left\Vert \va - \vb \right\Vert_2$.

%% file: sections/030prob_method.tex
\section{Proposed Method}
\label{sec:method}

\subsubsection{Problem Formulation} Given a graph $G$ and its node-wise feature matrix $\mX \in \R^{n \times f}$, learn node embeddings $\vh_u \in \R^d$ for each node $u \in \gV$ without any additional information (e.g. node class labels). 
The learned embeddings should be suitable for various downstream tasks, such as node classification and node clustering.

\subsubsection{\method{}} We propose \method{}, which consists of three main steps (\cref{fig:arch}). First, a GNN encoder $\textsc{Enc}(\cdot)$ takes the graph as input and produce node embeddings $\mH = \textsc{Enc}(\mX, \mA)$. Next, a multi-layer perceptron (MLP) predictor network $\textsc{Pred}(\cdot)$ takes the embeddings by GNN and produces a second set of node embeddings $\mZ = \textsc{Pred}(\mH)$. We then perform a clustering algorithm (e.g., $k$-means) on $\mZ$ to produce a set of clusterings $\gC$. It is worth noting that the clustering algorithm does not have to be differentiable. Finally, we compute a cluster validation index (CVI) on the cluster assignments and backpropagate to update the encoder's and predictor's parameters. After training, only the GNN encoder $\textsc{Enc}(\cdot)$ and its produced embeddings $\mH$ are used to perform downstream tasks, and the predictor network $\textsc{Pred}(\cdot)$ is discarded (similar to the prediction heads in non-contrastive learning work~\cite{bgrl,afgrl,selfgnn}).

\subsection{Training \method{}}
\label{subsec:training}
As aforementioned in \cref{subsec:cvis}, we evaluate the silhouette index~\cite{silhouette} and the VRC~\cite{vrc} as learning objectives. In order to use them effectively, we must make slight modifications to the loss functions. First, we must negate the functions since a higher score is better for both CVIs, and we typically want to minimize a loss function. Second, while $\sil(\gC) = 1$ and $\vrc(\gC) \rightarrow \infty$ are theoretically ideal, we find this is generally not true in practice. This is because the clustering method may miscluster some nodes and fully maximizing the CVIs will push the misclustered representations too close together, negatively impacting a classifier's ability to distinguish them. To bound the maximum values of our loss, we add $\silTar$ and $\vrcTar$ --- the target silhouette and VRC indices, respectively.
The silhouette-based and VRC-based losses are then defined as follows:%
\begin{align}
\label{eqn:losses}
    L_{\sil{}} = \left| \silTar - \sil{}(\gC) \right| \thinspace , \hspace{2.5em}
    L_{\svrc{}} = \left| \vrcTar - \vrc{}(\gC) \right| \thinspace ,
\end{align}%
where $\silTar \in [-1, 1]$ is the target silhouette index and $\vrcTar \in [0, \infty)$ is the target VRC. 

Upon careful inspection of \cref{eqn:silhouette}, we can observe that the computational complexity for the silhouette is $O(n^2)$, while the complexity of VRC is only $O(n c)$, where $c \ll n$. This is a critical weakness in using the silhouette, especially when the goal is to avoid a quadratic runtime (the typical drawback of contrastive methods). Backpropagating on this loss function would also result in quadratic memory usage because we have to store the gradients for each operation. To resolve this issue, we leverage the simplified silhouette~\cite{hruschka2004evolutionary}, which instead uses the centroid distance. The simplified silhouette has been shown to have competitive performance with the original silhouette~\cite{wang2017analysis} while being much faster --- running in $O(n c)$ time. As such, we also try the simplified silhouette, which can be written as:%
\begin{equation}
\label{eqn:simp_sil}
    s'(u) = \frac{b'(u) - a'(u)}{\max \{ a'(u), b'(u) \} } \thinspace,
\end{equation}%
where $i = \pcAssign{}_u$ is the cluster assignment for $u$ and
\begin{equation}
    a'(u) = \tDist{} (u, \lCent{i}) \thinspace, \hspace{2.5em}
    b'(u) = \min_{\gC_k \neq \gC_i} \tDist(u, \lCent{i}) \thinspace,
\end{equation}%
We use the same loss function as $L_{\sil{}}$ (\cref{eqn:losses}), simply substituting $s'(u)$ for $s(u)$ (see \cref{subsec:silhouette}), and name it $L_{\simp{}}$.

\subsection{Clustering Method}
\label{subsec:clustering_method}

\paragraph{\kmeans} We primarily focus on \kmeans{} clustering for this framework due to its fast linear runtime (although we do briefly explore using $k$-medoids in \cref{subsec:kmed_vs_kmeans} below). The goal of \kmeans{} is to minimize the sum of squared errors---also known as the inertia or within-cluster sum of squares. Formally, this can be written as:%
\begin{equation}
    \argmin_\gC \sum_{i=1}^c \sum_{x \in \gC_i} \tDist{} (\vx, \lCent{i}) \thinspace.
\end{equation}
Finding the optimal solution to this problem is NP-hard~\cite{dasgupta2008hardness}, but efficient approximation algorithms~\cite{minibatchKmeans,lloyds} have been developed that return an approximate solution in linear time (see \cref{subsec:fast_kmeans}). While \kmeans{} is fast, it is known to be heavily dependent on its initial centroid locations~\cite{bradley1998refining,kmeanspp}, which can be partially solved via repeated re-initialization and picking the clustering that minimizes the inertia.

Poor initialization is typically not a large issue in \kmeans{} use cases since the end goal is usually to compute a single clustering so we can simply repeat and re-initialize until we are satisfied. However, since we generate a new clustering once per epoch in \method{}, poor initialization can result in a large amount of variance between epochs.

To minimize the chance of poor centroid initialization occurring during training, we carry the cluster centroids over between epochs. The centroids will naturally update after running \kmeans{} since the embeddings $\mZ$ changes each epoch (after backpropagation with CVI-based loss).

\subsection{Theoretical Analysis}
\label{subsec:ml_equiv}
To gain a theoretical understanding of why our framework works, we compare it to Margin loss --- a fundamental contrastive loss function that has been shown to work well for self-supervised representation learning~\cite{graphsage,pinSAGE}. 
We show that CVI-based loss (especially silhouette loss) has some similarity to Margin based loss, which intuitively explains the success of CVI-based loss. In addition, we show that CVI-based loss has the advantages of (a) lower sensitivity to graph structure, and (b) no negative sampling required.

\subsubsection{Similarity analysis of CVI-based loss and Margin loss}
Both Margin loss and CVI-based loss fundamentally consist of two terms: one measuring the distance between neighbors/inter-cluster points and one measuring the distance between non-neighbors/inter-cluster points. This similarity allows us to analyze basic versions of our proposed silhouette loss and margin loss in the context of node classification and show that they are identical in both of their ideal cases. 
We further analyze the sensitivity of these losses with respect to various parameters of the graph to examine the advantages and disadvantages of our proposed method. To do this, we first define the mean silhouette and margin loss functions:
\begin{definition}[Mean Silhouette]
\label{def:simp_sil}
We define the mean silhouette loss (removing the hyperparameter $\tau_{\sil}$, focusing on the numerator (un-normalizing the index) and replacing $\min$ with the mean) as follows:%
\begin{equation}
L_{\meanSil}(u) = -(b_{\meanSil}(u) - a(u)) = a(u) - b_\meanSil{}(u) \thinspace,
\end{equation}%
where
\begin{equation}%
b_\meanSil(u) = \frac{1}{c-1} \sum_{j \neq i} \frac{1}{|\gC_j|}\sum_{v \in \gC_j} \tDist (\vh_u, \vh_v) \thinspace .
\end{equation}%
\end{definition}%
\begin{definition}[Margin Loss]
\label{def:margin}
We define margin loss as follows:
\begin{align}
\textsc{ml}(u) = &\frac{1}{|\neigh(u)|} \sum_{v \in \neigh(u)} \tDist(\vh_u,\vh_v) \\
&- \frac{1}{|\gV - \neigh(u) - \{ u \}|}\sum_{t \not \in \neigh(u)} \tDist(\vh_u,\vh_t) \thinspace . \nonumber
\end{align}%
\end{definition}%
It is worth noting that this margin loss differs from the max-margin loss traditionally used in graph SSL~\cite{pinSAGE}. We simplify it above in \cref{def:margin} by removing the $\max$ function for ease of analysis.

Let $\gL$ be the set of true class labels, and $\gL_u$ be the class label for a node $u \in \gV$.
We define the expected inter-class and intra-class distances as follows:%
    \begin{equation}
    \E \left[ \tDist(\vh_u, \vh_v) \right] =
        \begin{cases}
            \alpha, & \text{if } \gL_u = \gL_v \thinspace ; \\
            \beta, & \text{otherwise}, \\
        \end{cases}
    \end{equation}%
where $\alpha,\beta \in \R^+$. Next, let%
    \begin{equation}
    P\left( (u,v) \in \gE \right) =
        \begin{cases}
            p, & \text{if } \gL_u = \gL_v \thinspace ; \\
            q, & \text{otherwise}, \\
        \end{cases}%
    \end{equation}
i.e., $\gG$ follows a stochastic block model with a probability matrix $P \in [0, 1]^{c \times c}$ of the form:%
\begin{equation}
\label{eqn:sbm_matrix}
    P = \begin{bNiceMatrix}
        p       & q      &  q       & q \\
        q       & \Ddots &  \Ddots  & \Ddots \\
        q       & \Ddots &  \Ddots  & \Ddots \\
        q       & \Ddots &  \Ddots  & p \\
    \end{bNiceMatrix} \thinspace .
\end{equation}
Note that $q$ does not necessarily equal $1-p$. Finally, we define the inter-class clustering error rate $\epsilon$ and intra-class clustering error rate $\delta$ as follows:%
    \begin{align}
    P (\gC_u \neq \gC_v | \gL_u = \gL_v) &= \epsilon \thinspace ; \\
    P (\gC_u = \gC_v | \gL_u \neq \gL_v) &= \delta \thinspace .
    \end{align}%
\begin{claim}
Given the above assumptions, the expected value of the simplified silhouette loss approaches that of the margin loss as $p \rightarrow 1, q \rightarrow 0$, and $\epsilon, \delta \rightarrow 0$.
\end{claim}%
\begin{proof}
To find $\Ebr{L_s(u)}$, we first find $\Ebr{a(u)}$ and $\Ebr{b_s(u)}$:
\begin{align}
    \Ebr{a(u)} &= \frac{\alpha}{c} - \frac{\epsilon \alpha}{c} + \delta \beta - \frac{\delta \beta}{c} \thinspace; \\
    \Ebr{b_s(u)} &= \frac{\epsilon \alpha}{c} + \beta - \beta\delta - \frac{\beta}{c} + \frac{\delta \beta}{c} \thinspace; \\
    \Ebr{L_s(u)} &= \Ebr{a(u)} - \Ebr{b_s(u)} \thinspace \\
    &= - \frac{2 \epsilon \alpha}{c} - \frac{2 \delta \beta}{c} + \frac{\beta}{c} + \frac{\alpha}{c} - \beta + 2 \delta \beta \thinspace . \nonumber
\end{align}%
Next, we take its limit as $\epsilon,\delta \rightarrow 0$:
\begin{align}
\lim_{\epsilon,\delta \rightarrow 0} \left( - \frac{2 \epsilon \alpha}{c} - \frac{2 \delta \beta}{c} + \frac{\beta}{c} + \frac{\alpha}{c} - \beta + 2 \delta \beta \right) = \frac{\beta}{c} + \frac{\alpha}{c} - \beta \thinspace .
\end{align}

To find $\Ebr{\textsc{ml}(u)}$, we first find the expected value of its left and right sides:
\begin{align}
    \Ebr{\frac{1}{\neigh(u)} \sum_{v \in \neigh(u)} \tDist(\vh_u,\vh_v)} = \frac{\alpha p}{c} + \beta q - \frac{\beta q}{c} \thinspace; \\
    \Ebr{\frac{1}{|\gV - \neigh(u) - \{ u \}|}\sum_{t \not \in \neigh(u)} \tDist(\vh_u,\vh_t)} \\
    = \frac{\alpha}{c} - \frac{\alpha p}{c} + \beta - \beta q - \frac{\beta}{c} + \frac{\beta q}{c} \thinspace.
\end{align}%
Substituting them back in, we get
\begin{align}
    \Ebr{\textsc{ml}(u)} = \frac{2\alpha p}{c} + 2\beta q - \frac{2\beta q}{c} - \frac{\alpha}{c} + \frac{\beta}{c} - \beta \thinspace .
\end{align}
Taking its limit as $p \rightarrow 1, q \rightarrow 0$, we find
\begin{align}
\lim_{p \rightarrow 1, q \rightarrow 0} \left( \frac{2\alpha p}{c} + 2\beta q - \frac{2\beta q}{c} - \frac{\alpha}{c} + \frac{\beta}{c} - \beta \right) \\
= \frac{2 \alpha}{c} - \frac{\alpha}{c} + \frac{\beta}{c} - \beta = \frac{\alpha}{c} + \frac{\beta}{c} - \beta  \thinspace .\\
\therefore \lim_{p \rightarrow 1, q \rightarrow 0} \Ebr{\textsc{ml}(u)} = \lim_{\epsilon,\delta \rightarrow 0} \Ebr{L_s(u)} \thinspace .
\end{align}%
\end{proof}
\arxivOnly{The full derivation can be found in \cref{proof:full_sil}. }Since the two loss functions are identical in their ideal cases, one may wonder: \textit{Why not use margin loss instead?} Well, the silhouette-based loss has two key advantages: 

\subsubsection{Lower sensitivity to graph structure.} The margin loss is minimized as $p \rightarrow 1$ and $q \rightarrow 0$. However, $p$ and $q$ are attributes of the graph itself, making it difficult for a user to directly improve the performance of a model using that loss function. On the other hand, the mean silhouette depends on $\epsilon$ and $\delta$, the inter/intra-class clustering error rates, instead. Even on the same graph, a silhouette-based loss can likely be improved by either choosing a more suitable clustering method or distance metric. This greatly increases the flexibility of this loss function.

\subsubsection{No negative sampling.} Negative sampling is required for most graph contrastive methods and often requires either many samples~\cite{mvgrl} or carefully chosen samples~\cite{Ying2018HierarchicalPoolingb,understanding_neg_sampling}. This is costly, often costing quadratic time~\cite{bgrl}. The primary advantage of non-contrastive methods is that they avoid this step~\cite{gbt,bgrl}. The simplified silhouette avoids this issue by only working in the $n \times d$ embedding space instead of the $n \times n$ graph. It also contrasts node representations against centroid representations instead of against other nodes directly.

%% file: sections/040experiments.tex
\section{Experimental Evaluation}
\label{sec:experiments}

We evaluate 3 variants of \method{}: (a) \textbf{\silMethod{}} --- based on the silhouette loss in \cref{eqn:losses}, (b) \textbf{\vrcMethod{}} --- based on the VRC loss in \cref{eqn:losses}, and (c) \textbf{\simpMethod{}} --- based on the simplified silhouette loss in \cref{eqn:simp_sil}. We evaluate these variants on 5 datasets on node classification and thoroughly benchmark their memory usage and runtime. We then select the best-performing variant, \simpMethod{}, and evaluate its performance across 2 additional tasks: node clustering, and embedding similarity search.

\subsubsection{Node Classification}
\label{desc:eval_classification}
A common task for GNNs is to classify each node into one of several different classes. In the supervised setting, this is often explicitly optimized for during the training process since the GNN is typically trained with cross-entropy loss over the labels~\citep{gcn,graphsage} but this is not possible for graph SSL methods where we do not have the labels during the training of the GNN. As such, the convention~\cite{dgi,grace,gbt,bgrl,afgrl} is to train a logistic regression classifier on the frozen embeddings produced by the GNN.

Following previous works, we train a logistic regression model with $\ell_2$ regularization on the frozen embeddings produced by our encoder model. We compare against a variety of self-supervised baselines, including both GNN-based and non-GNN-based: DeepWalk~\cite{deepWalk14}, RandomInit~\cite{dgi}, DGI~\cite{dgi}, GMI~\cite{gmi}, MVGRL~\cite{mvgrl}, GRACE~\cite{grace}, G-BT~\cite{gbt}, AFGRL~\cite{afgrl}, and BGRL~\cite{bgrl}. We also evaluate our method against two supervised models: GCA~\cite{gca} and a GCN~\cite{gcn}. We follow \cite{bgrl,afgrl} and use an 80/10/10 train/validation/test, early stopping on the validation accuracy. We re-run AFGRL and BGRL using their published code and weights (where possible) on that split%
\footnote{The official \href{https://github.com/nerdslab/bgrl/blob/dec99f8c605e3c4ae2ece57f3fa1d41f350d11a9/bgrl/logistic\_regression\_eval.py\#L9}{BGRL implementation online} uses an 80/0/20 split compared to the 80/10/10 split mentioned in \cite{bgrl}, so we re-run their trained models on an 80/10/10 split. Most results are similar but we get slightly different results on \amazoncomputers{}{}.}%
. Finally, we use node2vec results from \cite{afgrl} and the reported results of the other baseline methods from their respective papers. See \cref{subsec:impl_details} for implementation details.

\subsubsection{Node Clustering}
\label{desc:node_clustering}
Following previous graph representation learning work~\cite{mvgrl,gala,afgrl}, we also evaluate \method{} on the task of node clustering. We fit a $k$-means model on the generated embeddings $\mH$ using the evaluation criteria from \cite{afgrl} --- NMI and cluster homogeneity. Following \cite{afgrl}, we re-run our model with different hyperparameters (the embeddings are not the same as node classification) and report the highest clustering scores. Due to computational resource constraints, we choose to only evaluate \simpMethod{}, the overall best-performing model. We report the scores of the baselines models from \cite{afgrl}.

\subsubsection{Similarity Search}
\label{desc:sim_search}
Following \cite{afgrl}, we evaluate our model on node similarity search. The goal of similarity search is to, given a query node $u$, return the $k$ nearest neighbors. In our setting, the goal is to return other nodes belonging to the same class as the query node. We evaluate the performance of each method by computing its Hits@$k$ --- the percentage of the top $k$ neighbors that belong to the same class. Similar to \cite{afgrl}, we evaluate our model every epoch and report the highest similarity search scores.  We use $k \in \{5,10\}$ and use the scores from \cite{afgrl} as baseline results.

\subsection{Evaluation Results}
\label{subsec:eval_results}

\subsubsection{Node Classification Performance} We show the node classification accuracy of our three proposed methods along with the baseline results in \cref{tab:classification}. \silMethod{} generally performs the best of all the evaluated methods, with the highest accuracy on 4/5 of the datasets (all except \wikics{}). \simpMethod{} generally performs similarly to \silMethod{}, with similar performance on all datasets except \wikics{} and \amazonphotos{}, and still outperforms the baselines on 4/5 datasets. \vrcMethod{} is the weakest-performing method of the three methods. It only outperforms baselines on 2/5 datasets. Since \simpMethod{}{} is much faster than \silMethod{} (see \cref{subsec:training,fig:runtime}) without sacrificing much performance, we focus on \simpMethod{} for the remainder of the evaluation tasks.

\subsubsection{Node Clustering Performance} We evaluate \simpMethod{} on node clustering and display the results in \cref{tab:clustering}. We find that it generally outperforms its baselines on 5/5 datasets in terms of NMI and 4/5 datasets in terms of homogeneity. \simpMethod{} and AFGRL~\cite{afgrl} both encourage a clusterable representations by utilizing \kmeans{} clustering as part of their respective training pipelines. 

\input{tables/clustering.tex}

\subsubsection{Similarity Search Performance} We evaluate \simpMethod{} on similarity search in \cref{tab:similarity_search}, where it roughly performs on par with AFGRL, the best-performing baseline. This is surprising, as AFGRL specifically optimizes for the similarity search task by using $k$-NN as one of the criteria to sample neighbors.

\begin{figure*}
\centering
    \subcaptionbox{%
        \amazoncomputers{}%
    }[.245\textwidth]
    {\includegraphics[width=0.245\textwidth]{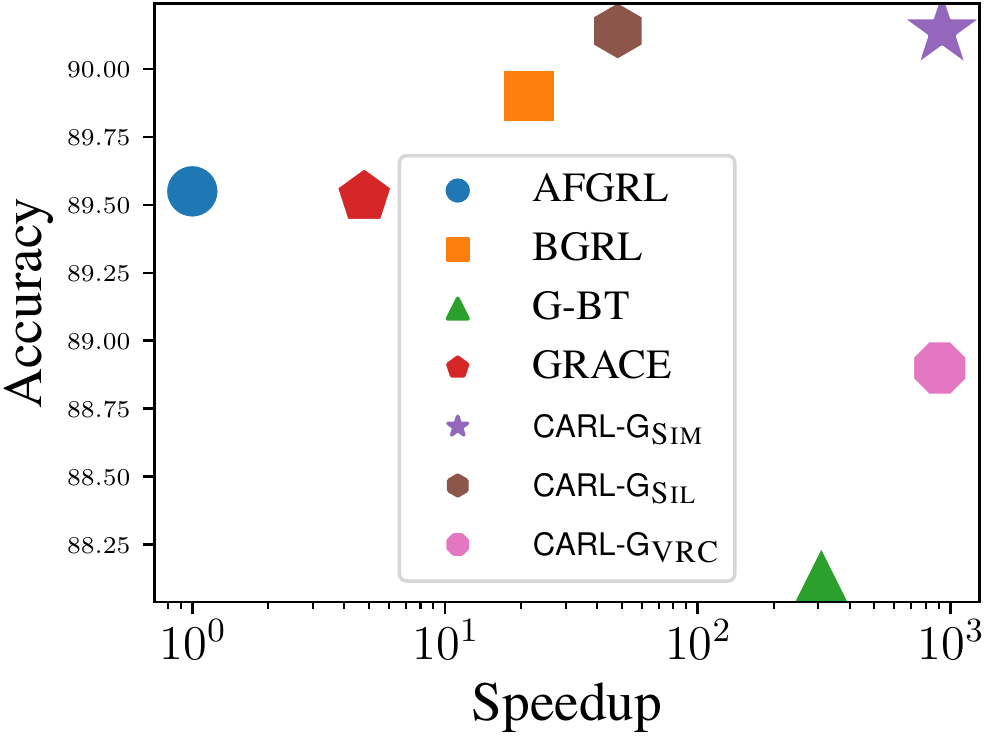}}%
    \subcaptionbox{%
        \coauthorcs{}%
    }[.245\textwidth]%
    {\includegraphics[width=0.245\textwidth]{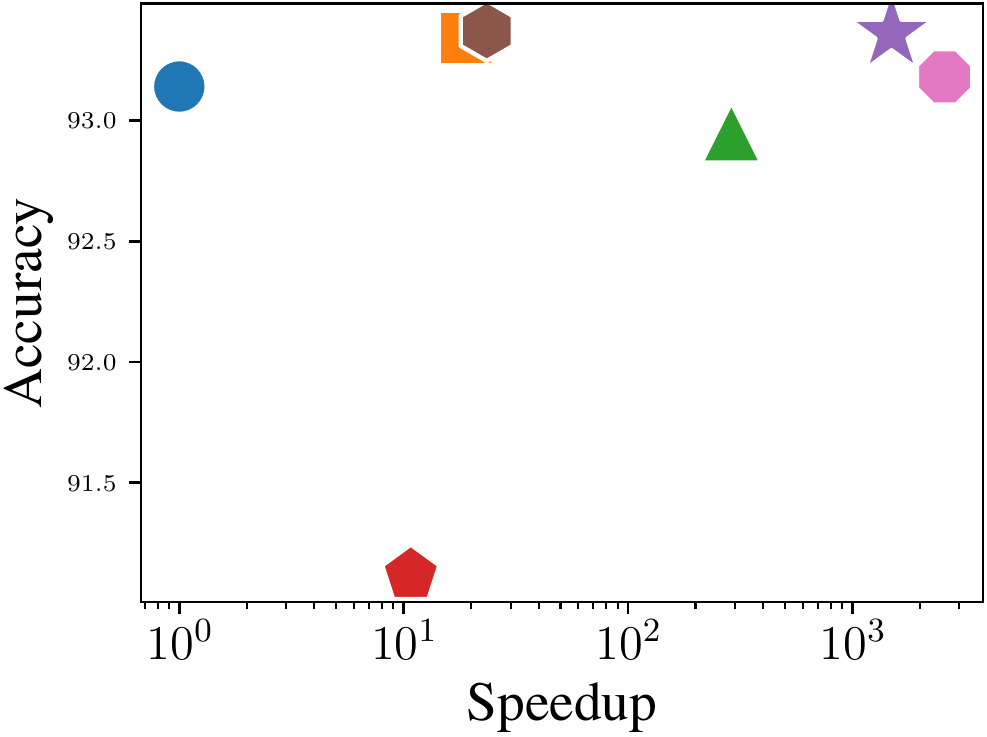}}%
        \subcaptionbox{%
        \coauthorphysics{}%
    }[.245\textwidth]%
    {\includegraphics[width=0.245\textwidth]{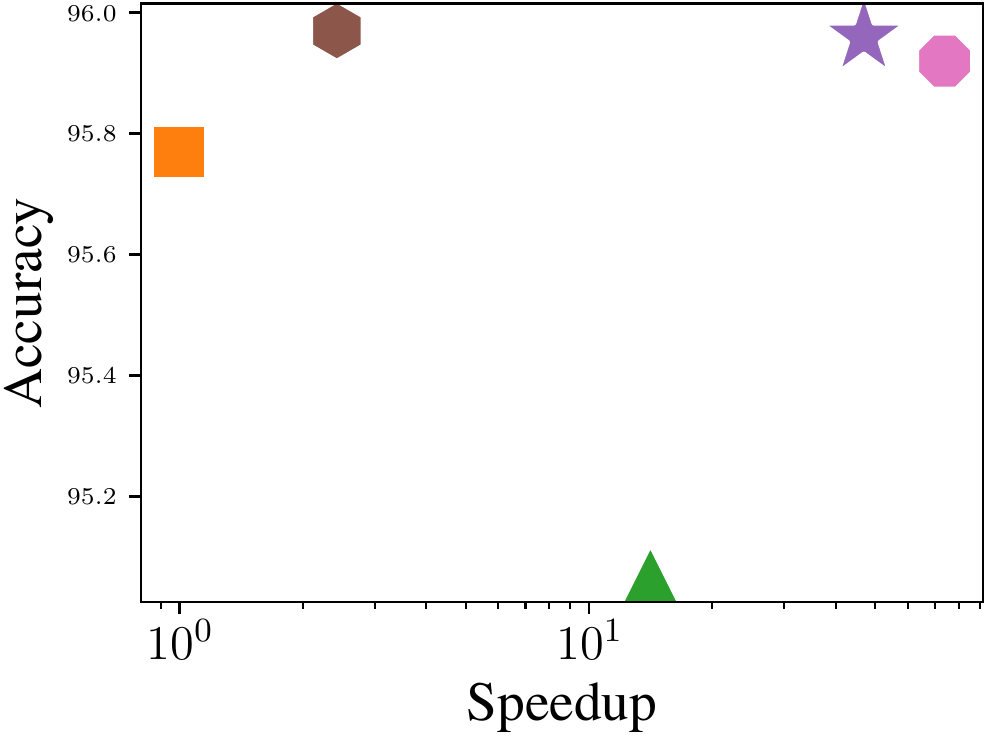}}%
        \subcaptionbox{%
        \wikics{}%
    }[.245\textwidth]%
    {\includegraphics[width=0.245\textwidth]{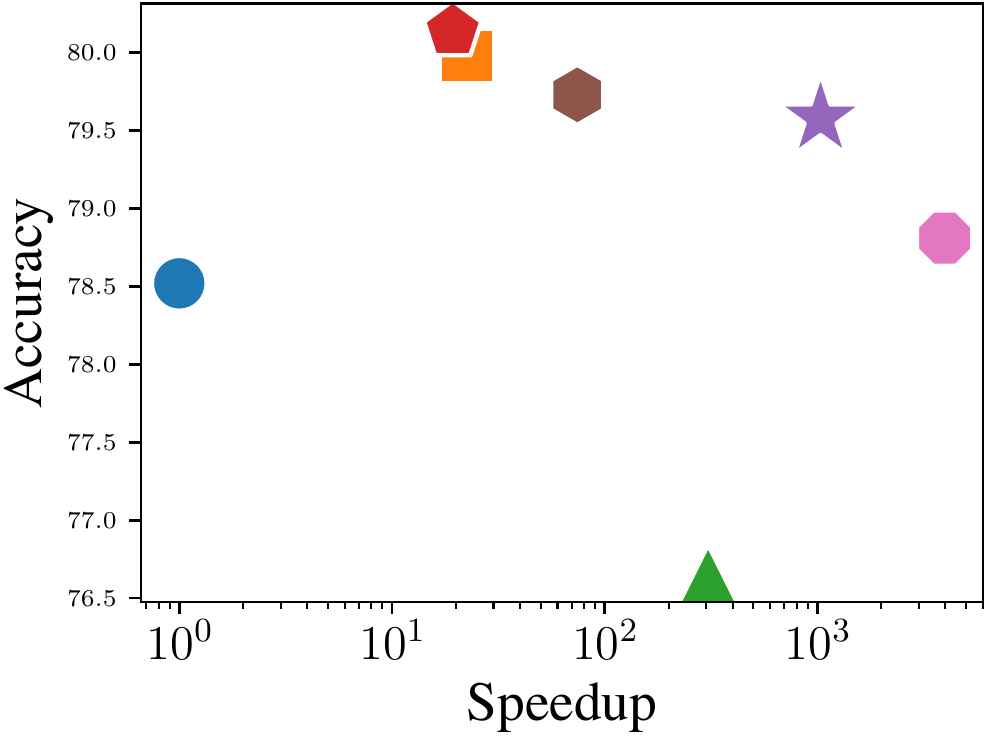}}%
    \caption{Runtime v.s. accuracy plots. \simpMethod{}, \silMethod{}, and \vrcMethod{} are our proposed methods. Speedup is relative to the slowest baseline (AFGRL). AFGRL and GRACE run out of memory on \coauthorphysics{}.}
    \label{fig:run_vs_acc}
    \Description{Plots showing the tradeoff between accuracy and speed for different proposed methods on different datasets. Generally, the simplified-silhouette-based CARL-G performs the best.}
    \vspace{-0.1in}
\end{figure*}

\subsection{Resource Benchmarking}
\label{subsec:benchmarking}
 We benchmark the 3 variants of our proposed method against BGRL~\cite{bgrl} (the best-performing baseline), AFGRL~\cite{afgrl} (the most recent baseline), G-BT~\cite{gbt} (the fastest baseline), and GRACE~\cite{grace} (a strong contrastive baseline). We time the amount of time it takes to train each of the best-performing node classification models. We remove all evaluation code and purely measure the amount of time it takes to train each method, taking care to synchronize all asynchronous GPU operations. We use the default values in the respective papers for AFGRL and BGRL: 5,000 epochs for AFGRL and 10,000 epochs for BGRL. We use 50 epochs for \method{}, although our method converges much faster in practice.%

We also measure the GPU memory usage of each method. 
We use the hyperparameters by the respective paper authors for each dataset, which is why the methods use different encoder sizes. Note that the encoder sizes greatly affect the runtime and memory usage of each of the models, so we report the layer sizes used in \cref{tbl:encoder_sizes}. Our benchmarking results can be found in \cref{fig:gpu_mem,fig:runtime}.

\input{tables/layer_sizes.tex}

\begin{figure*}
\centering
    \subcaptionbox{%
        Mean total training time. The y-axis is on a log. scale.
        \label{fig:runtime}%
    }[.48\textwidth]
    {\includegraphics[width=0.48\textwidth]{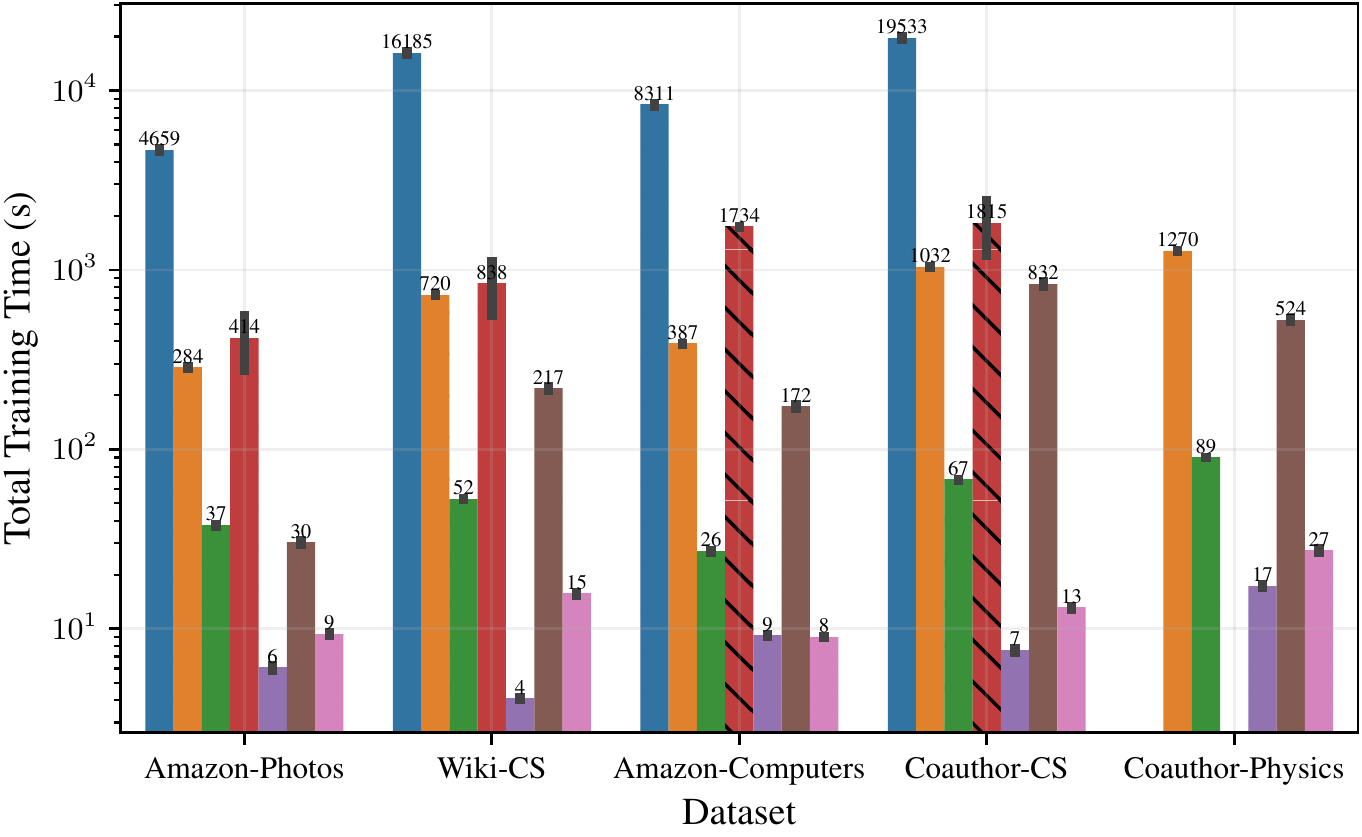}}%
    \hfill
    \subcaptionbox{%
        Max GPU memory usage.%
        \label{fig:gpu_mem}%
    }[.48\textwidth]%
    {\includegraphics[width=0.48\textwidth]{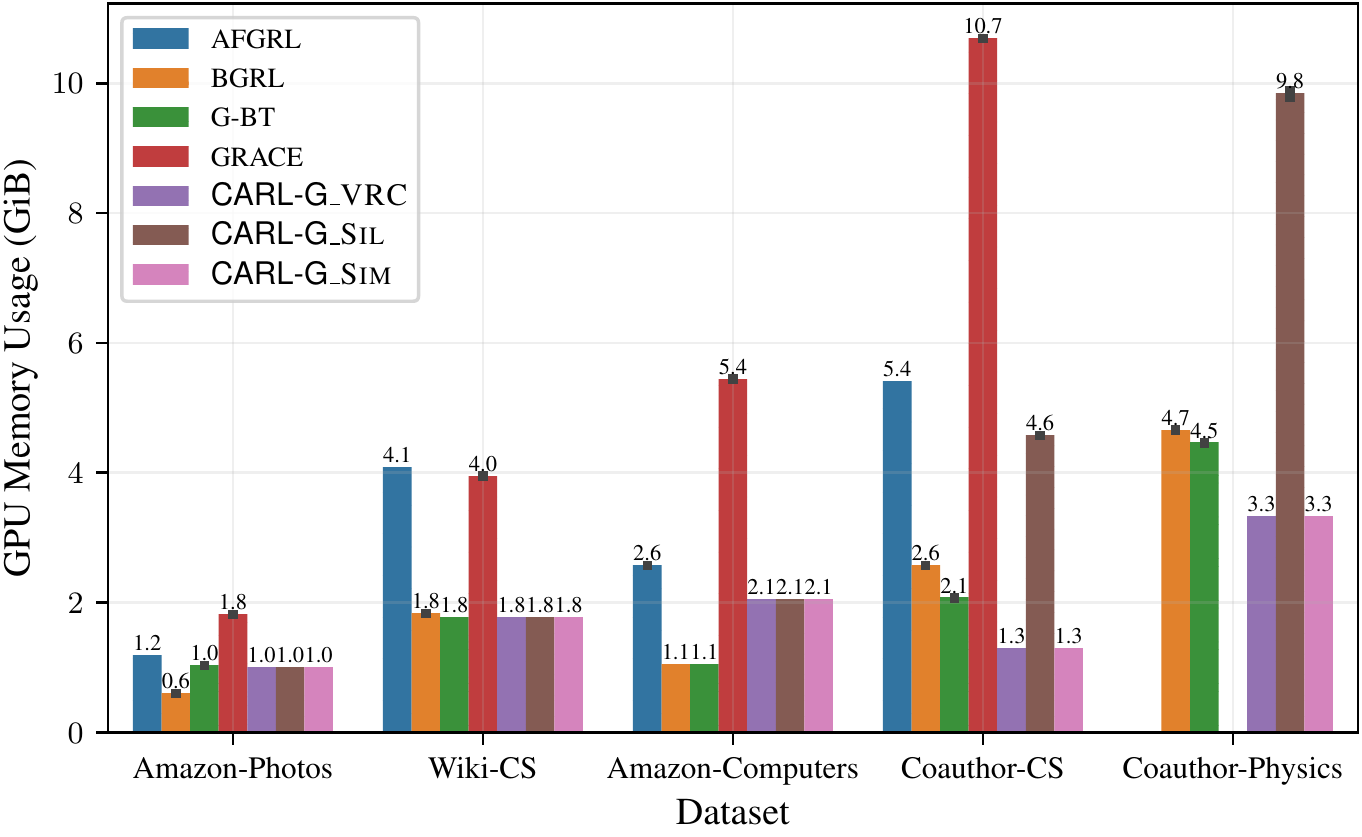}}%
    \caption{Mean total training time (left) and max GPU usage (right) for each model. \vrcMethod{} is the fastest with generally the least amount of memory used. \simpMethod{} uses the same amount of memory but is slightly slower. Note that not all of the baselines use the same encoder size---see \cref{tbl:encoder_sizes} for encoder sizes.}
    \Description{Mean total training time and GPU memory usage, where the VRC-based and simplified-silhouette-based methods perform the best.}
    \vspace{-0.1in}
\end{figure*}

\begin{figure}
\centering
    \subcaptionbox{%
        Acc. on \amazonphotos{}.%
        \label{fig:ablation_k_photos}%
    }[.48\columnwidth]
    {\includegraphics[width=0.48\columnwidth]{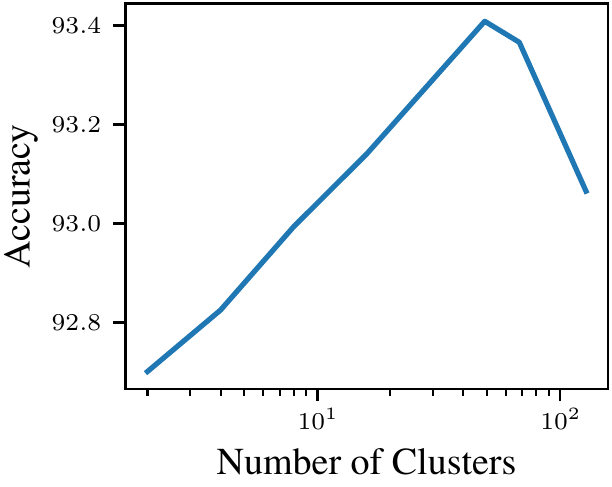}}%
    \hfill
    \subcaptionbox{%
        Acc. on \coauthorphysics{}.%
        \label{fig:ablation_k_phys}%
    }[.48\columnwidth]%
    {\includegraphics[width=0.48\columnwidth]{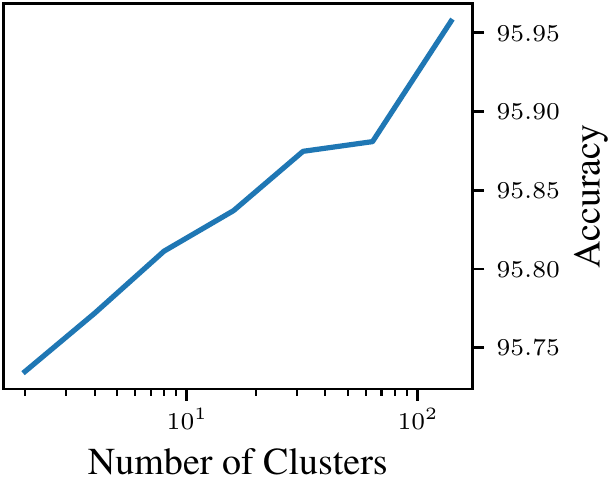}}%
    \caption{Node classification accuracy of \simpMethod{} on \amazonphotos{} and \coauthorphysics{} with a different number of clusters.}
    \Description{Node classification accuracy scores on amazon-photos and coauthor-physics with a different number of clusters. On amazon-photos, performance increases then drops, and on coauthor-physics, performance increases.}
\end{figure}

\begin{figure}[t]
    \centering
    \includegraphics[width=0.8\columnwidth]{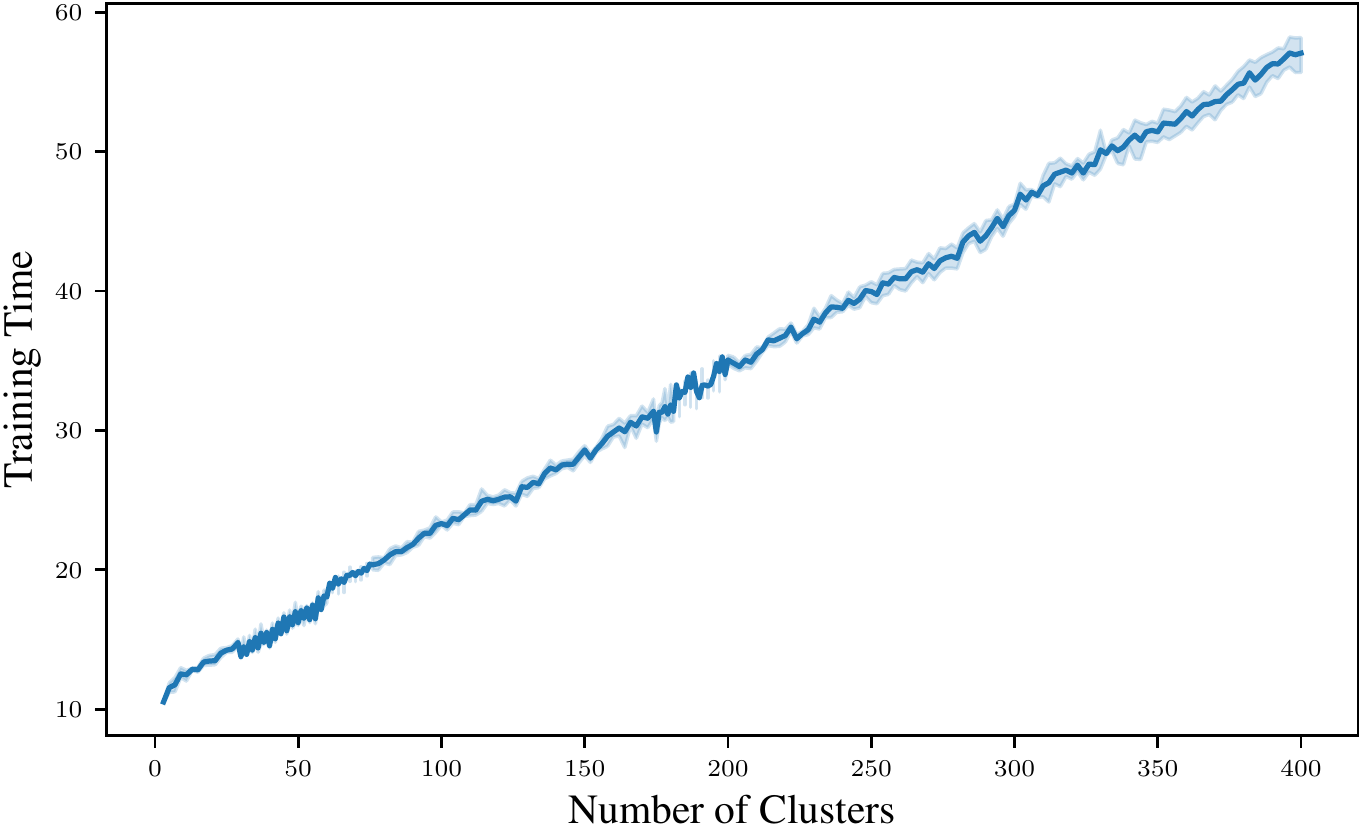}
    \caption{Training time versus number of clusters for \simpMethod{} on \coauthorphysics{}. As expected (see \cref{subsec:training}), the training time is linear with respect to the number of clusters.}
    \label{fig:train_time}
    \Description{Training time versus number of clusters for our method on coauthor-physics. The training time is linear with respect to the number of clusters.}
    \vspace{-0.2in}
\end{figure}%
\subsubsection{\methodname{} is fast.} In \cref{fig:runtime}, we show that \simpMethod{} is much faster than competing baselines, even in cases where the encoder is larger (see \cref{tbl:encoder_sizes}). BGRL is the best-performing node classification baseline, and \simpMethod{} is about 79\texttimes{} faster on \coauthorcs{}, and 57\texttimes{} faster on \coauthorphysics{}. AFGRL is by far the slowest method, requiring much longer to train.

\subsubsection{\method{} works with a fixed encoder size} We find that \method{} works well with a fixed encoder size (see \cref{tbl:encoder_sizes}). Unlike AFGRL, BGRL, GRACE, and G-BT, we fix the encoder size for \method{} across all datasets. This has practical advantages by allowing a user to fix the model size across datasets, thereby reducing the number of hyperparameters in the model. We observed that increasing the embedding size also increases the performance of our model across all datasets. This is not true for all of our baselines --- for example, \cite{afgrl} found that BGRL, GRACE, and GCA performance will often decrease in performance as embedding sizes increase. We limited our model embedding size to 256 for a fair comparison with other models.

\subsubsection{\simpMethod{} uses much less memory for the same encoder size.} When the encoder sizes of baseline methods are the same, \simpMethod{} uses much less memory than the baselines. The GPU memory usage of \simpMethod{} is also much lower than (about half) the memory usage of a BGRL model of the same size. This is because BGRL stores two copies of the encoder with different weights. The second encoder's weights gradually approach that of the first encoder during training but still takes up twice the space compared to single-encoder models like \simpMethod{} or G-BT~\cite{gbt}.

\subsubsection{\simpMethod{}'s runtime is linear with respect to the number of neighbors.} In \cref{subsec:training}, we mention that the runtime of \silMethod{}, the silhouette-based loss, is $O(n^2)$. This was the motivation for us to propose \simpMethod{} --- the simplified-silhouette-based loss which has an $O(nc)$ runtime instead. In \cref{fig:train_time}, we verify that this is indeed the case on \coauthorphysics{}, the largest of the 5 datasets.

\subsection{Ablation Studies}
\label{subsec:ablation}

We perform an ablation and sensitivity analysis on several aspects of our model. First, we examine the sensitivity of our model with respect to \numClusters{}---the number of clusters. Second, we examine the effect of using $k$-medoids instead of $k$-means. Finally, we try to inject more graph structural information during the clustering stage to see if we are losing any information.

\subsubsection{Effect of the number of clusters.}
\label{subsec:effect_of_clusters}
We perform sensitivity analysis on \numClusters{} --- the number of clusters (see \cref{fig:ablation_k_photos,fig:ablation_k_phys}). We find that, generally, the accuracy of our method goes up as the number of clusters increases. As the number of clusters continues to increase, the accuracy begins to drop. This implies that, much like traditional clustering~\cite{schubert2022stop}, there is some ``sweet spot'' for \numClusters{}. However, it is worth noting that this number does not directly correspond to the number of classes in the data and is much higher than $\numClusters{}$ for all of the datasets. DeepCluster~\cite{caron2018deep} also makes similar observations, where they find 10,000 clusters is ideal for ImageNet, despite there only being 1,000 labeled classes.
\input{tables/medoids_table.tex}

\subsubsection{$k$-medoids v.s. $k$-means.}
\label{subsec:kmed_vs_kmeans}
We study the effect of using $k$-medoids instead of \kmeans{} as our clustering algorithm. Both algorithms are partition-based clustering methods~\cite{xu2015comprehensive} and have seen optimizations in recent years~\cite{schubert2019faster,minibatchKmeans}. We find that the \kmeans{}-based \simpMethod{} generally performs better across all 4 of the evaluated datasets. The differences in node classification accuracy are shown in \cref{tab:kmed_effect}.

\subsubsection{Does additional information help?}
\label{subsec:additional_info_help}
It may appear as if we are losing graph information by working only with the embeddings. If this is the case, we should be able to improve the performance of our method by injecting additional information into the clustering step. We can do this by modifying the distance function of our clustering algorithm to the following:%
\begin{equation}
\textsc{Dist}(\vh_u, \vh_v) = \lambda \left\Vert \vh_u - \vh_v \right\Vert_2 + (1-\lambda) \mD_{u,v} \thinspace ,
\end{equation}%
where $\mD$ is the all-pairs shortest path (APSP) length matrix of $\gG$. This allows us to inject node neighborhood information into the clustering algorithm on top of the aggregation performed by the GNN. However, we find there is no significant change in performance for low $\lambda$ and performance decreases for high $\lambda$. This helps confirm the hypothesis that the GNN encoder is able to successfully embed a sufficient amount of structural data in the embedding.

\input{tables/classification.tex}
\input{tables/sim_search.tex}

\subsection{Implementation Details}
\label{subsec:impl_details}
For fair evaluation with other baselines, we elect to use a standard GCN~\cite{gcn} encoder. Our focus is on the overall framework rather than the architecture of the encoder. All of our baselines also use GCN layers. Following \cite{bgrl,afgrl}, we use two-layer GCNs for all datasets and use a two-layer MLP for the predictor network.%
We implement our model with PyTorch~\cite{pytorch} and PyTorch Geometric~\cite{pyg}. A copy of our code is publically available at \url{https://github.com/willshiao/carl-g}. We adapt the code from \cite{shiao2022link}, which contains implementations of BGRL~\cite{bgrl}, GRACE~\cite{grace}, and GBT~\cite{gbt} to use the split and downstream tasks from \cite{afgrl}. We also use the official implementation of AFGRL~\cite{afgrl}. We perform 50 runs of Bayesian hyperparameter optimization on each dataset and task for each of our 3 methods. The hyperparameters for our results are available at that link. All of our timing experiments were conducted on Google Cloud Platform using 16 GB NVIDIA V100 GPUs.\arxivOnly{ For more implementation details, please refer to \cref{app:additional_details}.}

\subsection{Limitations \& Future Work}
\label{subsec:limitations}

While our proposed framework has been shown to be highly effective in terms of both training speed and performance across the 3 tasks, there are also some limitations to our approach. One such limitation is that we use hard clustering assignments, i.e., each node is assigned to exactly one cluster. This can pose issues for multi-label datasets like the Protein-Protein Interaction (PPI)~\cite{ppi} graph dataset. One possible solution to this problem is to perform soft clustering and use a weighted average of CVIs for second/tertiary cluster assignments, but this would require major modifications to our method, and we reserve an exploration of this for future work.

%% file: tables/clustering.tex
\begin{table}[t]
    \centering
    \small
    \scale[0.97]{\begin{tabular}{c|c|ccc|cc}
        \toprule
    \multicolumn{2}{c|}{} & GRACE & GCA & BGRL & AFGRL & \simpMethod{} \\
        \midrule
        \multirow{2}{*}{\texttt{WikiCS}}     & NMI  & 0.428 & 0.337 & 0.397          & 0.413          & \textbf{0.471} \\ 
                                    & Hom. & 0.442 & 0.353 & 0.416          & 0.430          & \textbf{0.491} \\
        \midrule
        \multirow{2}{*}{\texttt{Computers}}  & NMI  & 0.479          & 0.528 & 0.536          & 0.552 & \textbf{0.558} \\ 
                                    & Hom. & 0.522          & 0.582 & 0.587          & 0.604 & \textbf{0.607} \\
        \midrule
        \multirow{2}{*}{\texttt{Photo}}      & NMI  & 0.651          & 0.644 & 0.684 & 0.656          & \textbf{0.701} \\ 
                                    & Hom. & 0.666          & 0.658 & 0.700 & 0.674          & \textbf{0.718} \\
        \midrule
        \multirow{2}{*}{\texttt{Co.CS}}      & NMI  & 0.756          & 0.762 & 0.773          & 0.786 & \textbf{0.790} \\
                                    & Hom. & 0.791          & 0.797 & 0.804          & \textbf{0.816} & 0.815 \\
        \midrule
        \multirow{2}{*}{\texttt{Co.Physics}} & NMI  & OOM            & OOM   & 0.557          & 0.729 & \textbf{0.771} \\
                                    & Hom. & OOM            & OOM   & 0.602          & 0.735 & \textbf{0.776} \\
        \bottomrule
    \end{tabular}}
    \caption{Node clustering performance in terms of cluster NMI and homogeneity. \simpMethod{} outperforms the baselines on 4/5 datasets.}
    \label{tab:clustering}
    \vspace{-0.3in}
\end{table}

%% file: tables/layer_sizes.tex
\begin{table}[!htp]
\centering
\resizebox{\columnwidth}{!}{
\begin{tabular}{lccccc}
\toprule
                & \texttt{Computers}    &       \texttt{Photos} &   \texttt{Co-CS}   &      \texttt{Co-Phy}      &    \texttt{Wiki} \\
\midrule
AFGRL           &            [512] &         [512] &      [1024] &              OOM &     [1024] \\
BGRL            &        [256,128] &     [256,128] &   [512,256] &        [256,128] &  [512,256] \\
G-BT            &        [256,128] &     [512,256] &   [512,256] &        [256,128] &  [512,256] \\
GRACE           &        [256,128] &     [256,128] &   [512,256] &        [256,128] &  [512,256] \\
\simpMethod{}   &        [512,256] &     [512,256] &   [512,256] &        [512,256] &  [512,256] \\
\silMethod{}    &        [512,256] &     [512,256] &   [512,256] &        [512,256] &  [512,256] \\
\vrcMethod{}    &        [512,256] &     [512,256] &   [512,256] &        [512,256] &  [512,256] \\
\bottomrule
\end{tabular}
}
\caption{GCN layer sizes used by the encoder for each method. The layer sizes greatly affect the amount of memory used by each model (shown in \cref{fig:gpu_mem}).}
\label{tbl:encoder_sizes}
\vspace{-0.3in}
\end{table}

%% file: tables/medoids_table.tex
\vspace{0.5em}
\begin{wrapfigure}[8]{R}{0.35\columnwidth}
\begin{flushright}
\vspace{-0.37in}
\begin{minipage}[r]{.35\columnwidth}
\begin{table}[H]
    \caption{$k$-medoids w/ \simpMethod{}.}
    \label{tab:kmed_effect}
    \vspace{-0.10in}
    \centering
    \resizebox{\columnwidth}{!}{
    \begin{tabular}{l|c}
        \toprule
        \textbf{Dataset}     &   \textbf{Accuracy $\Delta$} \\
        \midrule
        \texttt{Computers}  &   -1.41   \\
        \texttt{Co.CS}      &   -0.33   \\
        \texttt{Co.Phy}     &   -0.07   \\
        \texttt{Photos}     &   -0.11   \\
        \bottomrule
    \end{tabular}
    }
\end{table}
\end{minipage}
\end{flushright}
\end{wrapfigure}

%% file: tables/classification.tex
\begin{table*}[!ht]
\centering
\begin{tabular}{c|lccccc}
  \toprule
                   & \textbf{Method} & \wikics                   & \amazoncomputers & \amazonphotos    & \coauthorcs      & \coauthorphysics \\
  \midrule
  \multirow{4}{*}{Traditional}
  & Raw Features     & 71.98 $\pm$ 0.00          & 73.81 $\pm$ 0.00 & 78.53 $\pm$ 0.00 & 90.37 $\pm$ 0.00 & 93.58 $\pm$ 0.00 \\
  & node2vec~\cite{node2vec}   & 71.79 $\pm$ 0.05          & 84.39 $\pm$ 0.08 & 89.67 $\pm$ 0.12 & 85.08 $\pm$ 0.03 & 91.19 $\pm$ 0.04 \\
  & DeepWalk~\cite{deepWalk14} & 74.35 $\pm$ 0.06          & 85.68 $\pm$ 0.06 & 89.44 $\pm$ 0.11 & 84.61 $\pm$ 0.22 & 91.77 $\pm$ 0.15 \\
  & DeepWalk~\cite{deepWalk14} + Feat. & 77.21 $\pm$ 0.03          & 86.28 $\pm$ 0.07 & 90.05 $\pm$ 0.08 & 87.70 $\pm$ 0.04 & 94.90 $\pm$ 0.09 \\
  \midrule
  \multirow{8}{*}{GNN SSL} & Random-Init~\cite{dgi}      & 78.95 $\pm$ 0.58          & 86.46 $\pm$ 0.38 & 92.08 $\pm$ 0.48 & 91.64 $\pm$ 0.29 & 93.71 $\pm$ 0.29 \\
                     & DGI~\cite{dgi}     & 75.35 $\pm$ 0.14          & 83.95 $\pm$ 0.47 & 91.61 $\pm$ 0.22 & 92.15 $\pm$ 0.63 & 94.51 $\pm$ 0.09 \\
                     & GMI~\cite{gmi}     & 74.85 $\pm$ 0.08          & 82.21 $\pm$ 0.31 & 90.68 $\pm$ 0.17 & OOM              & OOM              \\
                     & MVGRL~\cite{mvgrl} & 77.52 $\pm$ 0.08          & 87.52 $\pm$ 0.11 & 91.74 $\pm$ 0.07 & 92.11 $\pm$ 0.12 & 95.33 $\pm$ 0.03 \\
                     & GRACE~\cite{grace} & \textbf{80.14} $\pm$ 0.48 & 89.53 $\pm$ 0.35 & 92.78 $\pm$ 0.45 & 91.12 $\pm$ 0.20 & OOM              \\
                     & G-BT~\cite{gbt}    & 76.65 $\pm$ 0.62          & 88.14 $\pm$ 0.33 & 92.63 $\pm$ 0.44 & 92.95 $\pm$ 0.17 & 95.07 $\pm$ 0.17 \\
                     & AFGRL~\cite{afgrl} & 78.52 $\pm$ 0.72          & 89.55 $\pm$ 0.36 & 92.91 $\pm$ 0.26 & 93.14 $\pm$ 0.23 & OOM              \\
                     & BGRL~\cite{bgrl}   & \underline{79.98} $\pm$ 0.10 & \underline{89.90} $\pm$ 0.19 & 93.17 $\pm$ 0.30 & 93.34 $\pm$ 0.13 & 95.77 $\pm$ 0.05 \\
  \midrule
  \multirow{3}{*}{Proposed}
  &\vrcMethod       & 78.81 $\pm$ 0.49          & 88.90 $\pm$ 0.39 & 93.31 $\pm$ 0.36 & 93.18 $\pm$ 0.31 & 95.92 $\pm$ 0.14 \\
  &\simpMethod      & 79.58 $\pm$ 0.60          & \textbf{90.14} $\pm$ 0.33 & \underline{93.37} $\pm$ 0.37 & \underline{93.36} $\pm$ 0.39 & \underline{95.96} $\pm$ 0.09 \\
  &\silMethod       & 79.73 $\pm$ 0.44          & \textbf{90.14} $\pm$ 0.34 & \textbf{93.44} $\pm$ 0.32 & \textbf{93.37} $\pm$ 0.33 & \textbf{95.97} $\pm$ 0.14 \\
  \midrule
  \multirow{2}{*}{Supervised}
  &GCA~\cite{gca}   & 78.35 $\pm$ 0.05          & 88.94 $\pm$ 0.15 & 92.53 $\pm$ 0.16 & 93.10 $\pm$ 0.01 & 95.73 $\pm$ 0.03 \\
  &Supervised GCN~\cite{gcn}  & 77.19 $\pm$ 0.12          & 86.51 $\pm$ 0.54 & 92.42 $\pm$ 0.22 & 93.03 $\pm$ 0.31 & 95.65 $\pm$ 0.16 \\
  \bottomrule
\end{tabular}
\caption{Table of node classification accuracy. Bolded entries indicate the highest accuracy for that dataset. Underlined entries indicate the second-highest accuracy. OOM indicates out-of-memory.}
\label{tab:classification}
\vspace{-0.25in}
\end{table*}

%% file: tables/sim_search.tex
\begin{table}[t]
    \centering
    \small
    \scale[0.94]{\begin{tabular}{c|c|cccc|c}
        \toprule
        \multicolumn{2}{c|}{} & GRACE & GCA & BGRL & AFGRL & \simpMethod{}  \\
        \midrule
        \multirow{2}{*}{WikiCS}     & Hits@5  & 0.775 & 0.779          & 0.774          & 0.781 & \textbf{0.789} \\
                                    & Hits@10 & 0.765 & 0.767 & 0.762          & 0.766          & \textbf{0.775} \\
        \midrule
        \multirow{2}{*}{Computers}  & Hits@5  & 0.874 & 0.883          & 0.895          & \textbf{0.897} & 0.881 \\
                                    & Hits@10 & 0.864 & 0.874          & 0.886          & \textbf{0.889} & 0.871 \\
        \midrule
        \multirow{2}{*}{Photo}      & Hits@5  & 0.916 & 0.911          & \textbf{0.925} & 0.924          & 0.922 \\
                                    & Hits@10 & 0.911 & 0.905          & \textbf{0.920} & 0.917          & 0.917 \\
        \midrule
        \multirow{2}{*}{Co.CS}      & Hits@5  & 0.910 & 0.913          & 0.911          & \textbf{0.918} & 0.916 \\
                                    & Hits@10 & 0.906 & 0.910          & 0.909          & \textbf{0.914} & \textbf{0.914} \\
        \midrule
        \multirow{2}{*}{Co.Physics} & Hits@5  & OOM    & OOM             & 0.950          & \textbf{0.953} & \textbf{0.953} \\
                                    & Hits@10 & OOM    & OOM             & 0.946          & 0.949 & \textbf{0.950} \\
        \bottomrule
    \end{tabular}}
    \caption{Performance on similarity search. Surprisingly, \method{} performs fairly well on this task, despite not being explicitly optimized for this task (unlike AFGRL, which uses KNN during training).}
    \label{tab:similarity_search}
    \vspace{-0.35in}
\end{table}

%% file: sections/050related.tex
\section{Additional Related Work}
\label{sec:related}

\paragraph{Deep Clustering.} A related, but distinct, area of work is deep clustering, which uses a neural network to directly aid in the clustering of data points~\cite{clusterWithDeepLearning}. %
However, the fundamental goal of deep clustering differs from graph representation learning in that the goal is to produce a clustering of the graph nodes rather than just representations of them. An example of this is DEC~\cite{xie2016unsupervised}, which uses a deep autoencoder with KL divergence to learn cluster centers, which are then used to cluster points with $k$-means.

\paragraph{Clustering for Representation Learning.} There exists work that uses clustering to learn embeddings~\cite{zhaoautogda,caron2018deep,yang2017towards}. Notably, DeepCluster~\cite{caron2018deep} trains a CNN with standard cross-entropy loss on pseudo-labels produced by \kmeans{}. Similarly, \cite{yang2017towards} simultaneously performs clustering and dimensionality reduction with a deep neural network. The key difference between those models and our proposed framework is that we use graph data and CVI-based losses instead of traditional supervised losses.

\paragraph{Clustering for Efficient GNNs.} There also exists work that uses clustering to speed up GNN training and inference. Cluster-GCN~\cite{chiang2019cluster} samples node blocks produced by graph clustering algorithms and speeds up GCN layers by limiting convolutions within each block for training and inference. However, it is worth noting that it computes a fixed clustering, rather than updating the clustering jointly with our model (unlike \methodname{}). %
FastGCN~\cite{chen2018fastgcn} does not explicitly cluster nodes but uses Monte Carlo importance sampling to similarly reduce neighborhood size and improve the speed of GCNs. 

\paragraph{Efficient \kmeans{}.}
\label{subsec:fast_kmeans}
Over the years, many variants and improvements to \kmeans{} have been proposed. The original method proposed to solve the \kmeans{} assignment problem was Lloyd's algorithm~\cite{lloyds}. Since then, several more efficient algorithms have been developed. \citet{bottou1994convergence} propose using stochastic gradient descent for finding a solution. \citet{minibatchKmeans} further builds on this work by proposing a \kmeans{} variant that uses mini-batching to dramatically speed up training. Finally, approximate nearest-neighbor search libraries like FAISS~\cite{johnson2019billion} %
allow for efficient querying of nearest neighbors, further speeding up training.

%% file: sections/060conclusions.tex
\section{Conclusion}
\label{sec:conclusions}
In this work, we are the first to introduce Cluster Validation Indexes in the context of graph representation learning. We propose a novel CVI-based framework and investigated trade-offs between different CVI variants. We find that
the loss function based on the simplified silhouette achieves the best overall performance to runtime ratio. It outperforms all baselines across 4/5 datasets in node classification and node clustering tasks, training up to 79\texttimes{} faster than the best-performing baseline. It also performs on-par with the best performing node similarity search baseline while training 1,500\texttimes{} faster. Moreover, to more comprehensively understand the effectiveness of \method{}, we establish a theoretical connection between the silhouette and the well-established margin loss.

%% file: sections/500acks.tex
The authors would like to thank UCR Research Computing and Ursa Major for the Google Cloud resources provided to support this research.
This research was also supported by the National Science Foundation under CAREER grant no. IIS 2046086 and CREST Center for Multidisciplinary Research Excellence in Cyber-Physical Infrastructure Systems (MECIS) grant no. 2112650, by the Agriculture and Food Research Initiative Competitive Grant no. 2020-69012-31914 from the USDA National Institute of Food and Agriculture and by the Combat Capabilities Development Command Army Research Laboratory and was accomplished under Cooperative Agreement Number W911NF-13-2-0045 (ARL Cyber Security CRA). The views and conclusions contained in this document are those of the authors and should not be interpreted as representing the official policies, either expressed or implied, of the Combat Capabilities Development Command Army Research Laboratory or the U.S. Government. The U.S. Government is authorized to reproduce and distribute reprints for Government purposes not withstanding any copyright notation here on.%

%% file: sections/900appendix.tex
\input{code/loss_proof.tex}

\input{tables/full_performance.tex}

\subsection{Meaning of Ideal Conditions}

Our analysis in \cref{proof:full_sil} aims to show that the more traditional margin-based losses and silhouette-based losses are sensitive to different parameters and their equivalence in the best-case scenario. Here, we briefly summarize what each of those ideal conditions means:

\begin{itemize}
    \item $p \rightarrow 1$: We approach the case where an edge exists between each node of the same class.
    \item $q \rightarrow 0$: We approach the case where an edge never exists between nodes of different classes.
    \item $\epsilon \rightarrow 0$: We approach the case where we always place two nodes in the same cluster if they are the same class.
    \item $\delta \rightarrow 0$: We approach the case where we never place two nodes in the same cluster if they are in different classes.
\end{itemize}

Essentially, the ideal case for a margin-loss GNN is $p \rightarrow 1$ and $q \rightarrow 0$. Conversely, the ideal case for \methodname{} is $\epsilon \rightarrow 0$, $\delta \rightarrow 0$. As we mentioned in \cref{subsec:ml_equiv}, silhouette-based loss relies on the clustering error rate rather than the inherent properties of the graph. We show that a margin-loss GNN is exactly equivalent to a mean-silhouette-loss GNN under the above conditions; however, it also follows that some equivalence can also be drawn between them for different non-ideal values of $p$, $q$, $\epsilon$, and $\delta$, but we feel such analysis is out of the scope of this work.

\subsection{Additional Experiment Details}
\label{app:additional_details}
We ran our experiments on a combination of local and cloud resources. All non-timing experiments were run on an NVIDIA RTX A4000 or V100 GPU, both with 16 GB of VRAM. All timing experiments were conducted on a Google Cloud Platform (GCP) instance with 12 CPU Intel Skylake cores, 64 GB of RAM, and a 16 GB V100 GPU. Accuracy means and standard deviations are computed by re-training the classifier on 5 different splits. The code and exact hyperparameters for this paper can be found online at \url{https://github.com/willshiao/carl-g}.

\subsection{Dataset Statistics}
\begin{table}[H]
\centering
\begin{tabular}{l||cccc} 
\toprule
\textbf{Dataset}    & \textbf{Nodes}    & \textbf{Edges}   & \textbf{Features}  & \textbf{Classes}  \\ \midrule
\wikics & 11,701 & 216,123 & 300 & 10\\
\coauthorcs &18,333 &163,788 &6,805 & 15\\   
\coauthorphysics &34,493 &495,924 &8,415 & 5 \\      
\amazoncomputers &13,752 &491,722 &767 & 10 \\
\amazonphotos &7,650 &238,162 &745 & 8 \\

\bottomrule
\end{tabular}
\caption{Statistics for the datasets used in our work.}
\label{tab:dset_statistic}
\end{table}

%% file: code/loss_proof.tex
\section{Appendix}
\label{appendix}

\subsection{Full Proof of Equivalency to Margin Loss}

\begin{proof}
\label{proof:full_sil}
For ease of analysis, we work with the simplified silhouette loss (\cref{def:simp_sil}) and the non-max margin loss (\cref{def:margin}).
Let $\gL$ be the set of class labels, and $\gL_u$ be the class label for node $u$. Let $\gC_u$ be the cluster assignment for node $u$, and $c = |C|$ be the number of clusters/classes. We define the expected inter-class and intra-class distances as follows:%
    \begin{equation}
    \E \left[ \tDist(\vh_u, \vh_v) \right] =
        \begin{cases}
            \alpha & \text{if } \gL_u = \gL_v \\
            \beta & \text{otherwise} \\
        \end{cases} \thinspace,
    \end{equation}%
where $\alpha,\beta \in \R^+$. Next, let%
    \begin{equation}
    P\left( (u,v) \in \gE \right) =
        \begin{cases}
            p & \text{if } \gL_u = \gL_v \\
            q & \text{otherwise} \\
        \end{cases} \thinspace, %
    \end{equation}
i.e., $\gG$ follows a stochastic block model with a probability matrix $P \in [0, 1]^{c \times c}$ of the form:%
\begin{equation}
\label{eqn:app_sbm_matrix}
    P = \begin{bNiceMatrix}
        p       & q      &  q       & q \\
        q       & \Ddots &  \Ddots  & \Ddots \\
        q       & \Ddots &  \Ddots  & \Ddots \\
        q       & \Ddots &  \Ddots  & p \\
    \end{bNiceMatrix} \thinspace .
\end{equation}
Note that $q$ does not necessarily equal $1-p$. We define the inter-class clustering error rate $\epsilon$ and intra-class clustering error rate $\delta$ as follows:%
    \begin{align}
    P (\gC_u \neq \gC_v | \gL_u = \gL_v) &= \epsilon \\
    P (\gC_u = \gC_v | \gL_u \neq \gL_v) &= \delta \thinspace .
    \end{align}%
To find $\Ebr{s_s(u)}$, we first find $\Ebr{a(u)}$ and $\Ebr{b_s(u)}$:
\begin{align}
    \Ebr{a(u)} &= \Ebr{\frac{1}{| \gC_i | - 1} \sum_{v \in (\gC_i - \{ u\})} \tDist(\vh_u, \vh_v)} \\
    &= \Ebrs{v}{\tDist(\vh_u, \vh_v) \vert \gC_u = \gC_v} \\
    &= P (\gL_u = \gL_v) \cdot P (\gC_u = \gC_v \vert \gL_u = \gL_v) \\
    &\hspace{0.7em} \cdot \Ebrs{v} {\tDist(\vh_u, \vh_v) \vert \gC_u = \gC_v \land \gL_u = \gL_v } \nonumber \\
    &\hspace{0.7em} + P (\gL_u \neq \gL_v) \cdot P (\gC_u = \gC_v \vert \gL_u \neq \gL_v) \nonumber \\
    &\hspace{0.7em} \cdot \Ebrs{v} {\tDist(\vh_u, \vh_v ) \vert \gC_u = \gC_v \land \gL_u \neq \gL_v } \nonumber \\%
    &= \left( \frac{1}{c} \right) (1-\epsilon) \alpha + \left(1 - \frac{1}{c} \right) \delta \beta \\
    &= \frac{\alpha}{c} - \frac{\epsilon \alpha}{c} + \delta \beta - \frac{\delta \beta}{c}
\end{align}%
and%
\begin{align}
    \Ebr{b_s(u)} &= \Ebr{\frac{1}{c-1} \sum_{j \neq i} \frac{1}{|C_j|}\sum_{v \in \gC_j} \tDist (\vh_u, \vh_v)} \\
    &= \Ebrs{j \neq i}{\frac{1}{|C_j|}\sum_{v \in \gC_j} \tDist (\vh_u, \vh_v)} \\
    &= \Ebrs{v}{\tDist(\vh_u, \vh_v \vert \gC_u \neq \gC_v)} \\
    &= P (\gL_u = \gL_v) \cdot P (\gC_u \neq \gC_v \vert \gL_u = \gL_v) \cdot \\
    &\hspace{0.7em} \Ebrs{v} {\tDist(\vh_u, \vh_v) \vert \gC_u \neq \gC_v \land \gL_u = \gL_v } \nonumber \\
    &\hspace{0.7em} + P (\gL_u \neq \gL_v) \cdot P (\gC_u \neq \gC_v \vert \gL_u \neq \gL_v) \nonumber \\
    &\hspace{0.7em} \cdot \Ebrs{v} {\tDist(\vh_u, \vh_v ) \vert \gC_u \neq \gC_v \land \gL_u \neq \gL_v } \nonumber \\
    &= \left( \frac{1}{c} \right) \epsilon \alpha  + \left(1 - \frac{1}{c} \right) (1 - \delta) \beta \\
    &= \frac{\epsilon \alpha}{c} + \beta \left( 1 - \delta -\frac{1}{c} + \frac{\delta}{c} \right) \\
    &= \frac{\epsilon \alpha}{c} + \beta - \beta\delta - \frac{\beta}{c} + \frac{\delta \beta}{c} \thinspace .
\end{align}%
Now, we can find $\Ebr{s_s(u)}$:%
\begin{align}
    \Ebr{s_s(u)} &= \Ebr{b_s(u)} - \Ebr{a(u)} \\
    &= \frac{\epsilon \alpha}{c} + \beta - \beta\delta - \frac{\beta}{c} + \frac{\delta \beta}{c} - \left( \frac{\alpha}{c} - \frac{\epsilon \alpha}{c} + \delta \beta - \frac{\delta \beta}{c} \right) \\
    &= \frac{\epsilon \alpha}{c} + \beta - \beta\delta - \frac{\beta}{c} + \frac{\delta \beta}{c} - \frac{\alpha}{c} + \frac{\epsilon \alpha}{c} - \delta \beta + \frac{\delta \beta}{c} \\
    &= \frac{2 \epsilon \alpha}{c} + \frac{2 \delta \beta}{c} - \frac{\beta}{c} - \frac{\alpha}{c} + \beta - 2 \delta \beta \thinspace .
\end{align}

Taking its limit as $\epsilon,\delta \rightarrow 0$, we find
\begin{align}
\lim_{\epsilon,\delta \rightarrow 0} \left( - \frac{2 \epsilon \alpha}{c} - \frac{2 \delta \beta}{c} + \frac{\beta}{c} + \frac{\alpha}{c} - \beta + 2 \delta \beta \right) = \frac{\beta}{c} + \frac{\alpha}{c} - \beta \thinspace .
\end{align}

We similarly break down the margin loss into two terms:
\begin{align}
&\Ebr{\frac{1}{\neigh(u)} \sum_{v \in \neigh(u)} \tDist(\vh_u,\vh_v)} \\
&= \Ebrs{v}{\tDist (\vh_u, \vh_v) \vert (u,v) \in \gE } \\
&= P(\gL_u = \gL_v) \cdot P((u,v) \in \gE \vert \gL_u = \gL_v) \\
&\hspace{0.7em} \cdot \Ebrs{v}{\tDist (\vh_u, \vh_v) \vert (u,v) \in \gE \land \gL_u = \gL_v } \nonumber \\
&\hspace{0.7em} + P(\gL_u \neq \gL_v) \cdot P((u,v) \in \gE \vert \gL_u \neq \gL_v) \nonumber \\
&\hspace{0.7em} \cdot \Ebrs{v}{\tDist (\vh_u, \vh_v) \vert (u,v) \in \gE \land \gL_u \neq \gL_v } \nonumber \\
&= \left( \frac{1}{c} \right) \alpha p + \left( 1 - \frac{1}{c} \right) \beta q \\
&= \frac{\alpha p}{c} + \beta q - \frac{\beta q}{c}
\end{align}%
and
\begin{align}
&\Ebr{\frac{1}{|\gV - \neigh(u) - \{ u \}|}\sum_{t \not \in \neigh(u)} \tDist(\vh_u,\vh_t)} \\
&= \Ebrs{v}{\tDist(\vh_u, \vh_v) \vert (u,v) \not \in \gE} \\
&= P(\gL_u = \gL_v) \cdot P((u,v) \not \in \gE \vert \gL_u = \gL_v) \\
&\hspace{0.7em} \cdot \Ebrs{v}{\tDist (\vh_u, \vh_v) \vert (u,v) \not \in \gE \land \gL_u = \gL_v } \nonumber \\
&\hspace{0.7em} + P(\gL_u \neq \gL_v) \cdot P((u,v) \not \in \gE \vert \gL_u \neq \gL_v) \nonumber \\
&\hspace{0.7em} \cdot \Ebrs{v}{\tDist (\vh_u, \vh_v) \vert (u,v) \not \in \gE \land \gL_u \neq \gL_v } \nonumber \\
&= \left( \frac{1}{c} \right) (1-p) \alpha + \left( 1 - \frac{1}{c} \right) (1-q) \beta \\
&= \frac{\alpha}{c} - \frac{\alpha p}{c} + \beta - \beta q - \frac{\beta}{c} + \frac{\beta q}{c} \thinspace .
\end{align}%
Re-substituting the terms into the margin loss equation, we get
\begin{align}
 &\E \left[ \frac{1}{|\neigh(u)|} \sum_{v \in \neigh(u)} \tDist(\vh_u,\vh_v) \right. \\
 &\hspace{0.7em} \left. - \frac{1}{|\gV - \neigh(u) - \{ u \}|}\sum_{t \not \in \neigh(u)} \tDist(\vh_u,\vh_t) \right] \nonumber \\
 &= \Ebr{\frac{1}{|\neigh(u)|} \sum_{v \in \neigh(u)} \tDist(\vh_u,\vh_v)} \\
 &\hspace{0.7em} - \Ebr{\frac{1}{|\gV - \neigh(u) - \{ u \}|}\sum_{t \not \in \neigh(u)} \tDist(\vh_u,\vh_t)} \nonumber \\
 &= \frac{\alpha p}{c} + \beta q - \frac{\beta q}{c} - \left( \frac{\alpha}{c} - \frac{\alpha p}{c} + \beta - \beta q - \frac{\beta}{c} + \frac{\beta q}{c} \right) \\
 &=\frac{\alpha p}{c} + \beta q - \frac{\beta q}{c} - \frac{\alpha}{c} + \frac{\alpha p}{c} - \beta + \beta q + \frac{\beta}{c} - \frac{\beta q}{c} \\
 &= \frac{2\alpha p}{c} + 2\beta q - \frac{2\beta q}{c} - \frac{\alpha}{c} + \frac{\beta}{c} - \beta \thinspace .
\end{align}%
Taking its limit as $p \rightarrow 1, q \rightarrow 0$:%
\begin{align}
\lim_{p \rightarrow 1, q \rightarrow 0} \left( \frac{2\alpha p}{c} + 2\beta q - \frac{2\beta q}{c} - \frac{\alpha}{c} + \frac{\beta}{c} - \beta \right) \\
= \frac{2 \alpha}{c} - \frac{\alpha}{c} + \frac{\beta}{c} - \beta = \frac{\alpha}{c} + \frac{\beta}{c} - \beta \\
\therefore \lim_{p \rightarrow 1, q \rightarrow 0} \Ebr{\textsc{ml}(u)} = \lim_{\epsilon,\delta \rightarrow 0} \Ebr{L_s(u)} 
\thinspace .
\end{align}
\end{proof}

%% file: tables/full_performance.tex
\begin{table*}
\begin{tabular}{l|lcccc}
Method & Dataset & Max GPU Memory & Mean CPU Memory & Training Time & Layer Sizes \\
\toprule
\multirow[c]{4}{*}{AFGRL} & \amazoncomputers{} & 2,637 & 2,671 & 8,311.94 & [512] \\
& \amazonphotos{} & 1,221 & 2,615 & 4,659.91 & [512] \\
& \coauthorcs{} & 5,537 & 3,038 & 19,533.74 & [1024] \\
& \wikics{} & 4,177 & 2,647 & 16,185.56 & [1024] \\
\midrule
\multirow[c]{5}{*}{BGRL} & \amazoncomputers{} & 1,081 & 2,289 & 387.03 & [256,128] \\
& \amazonphotos{} & 615 & 2,267 & 284.29 & [256,128] \\
& \coauthorcs{} & 2,637 & 2,722 & 1,032.14 & [512,256] \\
& \coauthorphysics{} & 4,769 & 3,362 & 1,270.93 & [256,128] \\
& \wikics{} & 1,877 & 2,240 & 720.37 & [512,256] \\
\midrule
\multirow[c]{5}{*}{\simpMethod{}} & \amazoncomputers{} & 2,100 & 2,910 & 8.97 & [512,256] \\
& \amazonphotos{} & 1,032 & 2,875 & 9.29 & [512,256] \\
& \coauthorcs{} & 1,325 & 3,352 & 13.07 & [512,256] \\
& \coauthorphysics{} & 3,405 & 3,998 & 27.11 & [512,256] \\
& \wikics{} & 1,816 & 2,857 & 15.64 & [512,256] \\
\midrule
\multirow[c]{5}{*}{\silMethod{}} & \amazoncomputers{} & 2,100 & 3,435 & 172.26 & [512,256] \\
& \amazonphotos{} & 1,032 & 3,320 & 30.13 & [512,256] \\
& \coauthorcs{} & 4,682 & 4,273 & 832.25 & [512,256] \\
& \coauthorphysics{} & 10,074 & 4,882 & 524.27 & [512,256] \\
& \wikics{} & 1,816 & 3,392 & 217.63 & [512,256] \\
\midrule
\multirow[c]{5}{*}{\vrcMethod{}} & \amazoncomputers{} & 2,100 & 2,875 & 9.15 & [512,256] \\
& \amazonphotos{} & 1,032 & 2,843 & 6.04 & [512,256] \\
& \coauthorcs{} & 1,325 & 3,320 & 7.57 & [512,256] \\
& \coauthorphysics{} & 3,405 & 3,964 & 17.22 & [512,256] \\
& \wikics{} & 1,816 & 2,826 & 4.08 & [512,256] \\
\bottomrule
\end{tabular}
\caption{Performance of various methods.}
\end{table*}